\documentclass[lettersize,journal]{ieeetran}

\IEEEoverridecommandlockouts                         

\usepackage{times}

\usepackage{multicol}
\usepackage[breaklinks=true,colorlinks,bookmarks=false]{hyperref}
\usepackage{amsthm}
\usepackage{adjustbox}
\usepackage{stackengine}
\usepackage{makecell}
\usepackage{caption}

\usepackage{comment}
\usepackage{siunitx}
\usepackage{relsize}
\usepackage{ifthen}
\usepackage[colorinlistoftodos]{todonotes}
\usepackage{pifont}

\usepackage[vlined,ruled,linesnumbered]{algorithm2e}
\usepackage{graphics} 
\usepackage{rotating}
\usepackage{color}
\usepackage{tcolorbox}
\usepackage{enumerate}
\usepackage[T1]{fontenc}
\usepackage{psfrag}
\usepackage{epsfig} 
\usepackage{booktabs}
\usepackage{graphicx,url}
\usepackage{multirow}
\usepackage{array}
\usepackage{latexsym}
\usepackage{amsfonts}
\usepackage{amsmath}
\usepackage{amssymb}
\usepackage{mathtools}
\usepackage{xstring}
\usepackage[noend]{algorithmic}
\usepackage{makecell}
\usepackage{xcolor}
\usepackage{booktabs}
\usepackage{siunitx}
\usepackage{threeparttable}
\usepackage{prettyref}
\usepackage{flexisym}
\usepackage{bigdelim}
\usepackage{breqn} 
\usepackage{listings}

\usepackage{enumitem}
\usepackage{xspace}
\usepackage{bm}
\graphicspath{{./figures/}}
\usepackage{tikz}
\usetikzlibrary{matrix,calc}
\usepackage{tabularx}
\usepackage{subcaption}
\usepackage{float}

\usepackage{pgfplots}
\usepackage{pgfplotstable}
\usepgfplotslibrary{groupplots} 
\pgfplotsset{compat=1.18}       

\usepackage{mdwlist}

\let\itemize\undefined
\makecompactlist{itemize}{stditemize}

\newrefformat{prob}{Problem\,\ref{#1}}
\newrefformat{def}{Definition\,\ref{#1}}
\newrefformat{sec}{Section\,\ref{#1}}
\newrefformat{sub}{Section\,\ref{#1}}
\newrefformat{prop}{Proposition\,\ref{#1}}
\newrefformat{app}{Appendix\,\ref{#1}}
\newrefformat{alg}{Algorithm\,\ref{#1}}
\newrefformat{cor}{Corollary\,\ref{#1}}
\newrefformat{thm}{Theorem\,\ref{#1}}
\newrefformat{lem}{Lemma\,\ref{#1}}
\newrefformat{fig}{Fig.\,\ref{#1}}
\newrefformat{tab}{Table\,\ref{#1}}
\newrefformat{assump}{Assumption\,\ref{#1}}

\newtheorem{theorem}{Theorem}

\newtheorem{definition}{Definition}

\newtheorem{remark}{Remark}

\newcommand{\bdmath}{\begin{dmath}}
\newcommand{\edmath}{\end{dmath}}
\newcommand{\beq}{\begin{equation}}
\newcommand{\eeq}{\end{equation}}
\newcommand{\bdm}{\begin{displaymath}}
\newcommand{\edm}{\end{displaymath}}
\newcommand{\bea}{\begin{eqnarray}}
\newcommand{\eea}{\end{eqnarray}}
\newcommand{\beal}{\beq \begin{array}{ll}}
\newcommand{\eeal}{\end{array} \eeq}
\newcommand{\beas}{\begin{eqnarray*}}
\newcommand{\eeas}{\end{eqnarray*}}
\newcommand{\ba}{\begin{array}}
\newcommand{\ea}{\end{array}}
\newcommand{\bit}{\begin{itemize}}
\newcommand{\eit}{\end{itemize}}
\newcommand{\ben}{\begin{enumerate}}
\newcommand{\een}{\end{enumerate}}

\newcommand{\calX}{{\cal X}}

\newcommand{\calZ}{{\cal Z}}

\newcommand{\ie}{\emph{i.e.,}\xspace}

\renewcommand{\boldsymbol}[1]{{\bm #1}}

\newcommand{\hide}[1]{}

\newcommand{\hiddenText}{{\color{gray} hidden text.}}
\newcommand{\hideWithText}[1]{\hiddenText}

\newcommand{\tran}{^{\mathsf{T}}}

\newcommand{\zero}{{\mathbf 0}}

\newcommand{\Real}[1]{ { {\mathbb R}^{#1} } }

\newcommand{\scenario}[1]{{\smaller \sf#1}\xspace}

\newcommand{\blue}[1]{{\color{blue}#1}}

\definecolor{best}{HTML}{1B7837}
\definecolor{second}{HTML}{5AAE61}
\newcommand{\best}[1]{\textbf{\textcolor{best}{#1}}}

\newcommand{\linkToPdf}[1]{\href{#1}{\blue{(pdf)}}}
\newcommand{\linkToPpt}[1]{\href{#1}{\blue{(ppt)}}}
\newcommand{\linkToCode}[1]{\href{#1}{\blue{(code)}}}
\newcommand{\linkToWeb}[1]{\href{#1}{\blue{(web)}}}
\newcommand{\linkToVideo}[1]{\href{#1}{\blue{(video)}}}
\newcommand{\linkToMedia}[1]{\href{#1}{\blue{(media)}}}
\newcommand{\award}[1]{\xspace} 

\newcommand{\bmat}{\left[ \begin{array}}
\newcommand{\emat}{\end{array}\right]}

\newcommand{\crisp}{\scenario{CRISP}}

\begin{document}

\title{
Fast and Safe Trajectory Optimization for Mobile
Manipulators With Neural Configuration Space
Distance Field
}

\author{Yulin Li$^{1,2}$, Zhiyuan Song$^{1}$, Yiming Li$^{3}$, Zhicheng Song$^{1}$, Kai Chen$^{1}$, Chunxin Zheng$^{1}$, \\ Zhihai Bi$^{1}$, Jiahang Cao$^{4}$, Sylvain Calinon$^{3}$, Fan Shi$^{2}$, and Jun Ma$^{1}$ 
\thanks{$^{1}$Robotics and Autonomous Systems Thrust, The Hong Kong University of Science and Technology (Guangzhou), Guangzhou, China.
}%
\thanks{$^{2}$Department of Electrical and Computer
Engineering, National University of Singapore, Singapore.}
\thanks{$^{3}$Idiap Research Institute, Martigny, Switzerland
and EPFL, Lausanne, Switzerland.}
\thanks{$^{4}$Institute of Data Science, The University of Hong Kong, Hong Kong SAR, China.}
}



\markboth{}{}

\maketitle
\renewcommand{\thefootnote}{\fnsymbol{footnote}}
\begin{abstract}
Mobile manipulators promise agile, long-horizon behavior by coordinating base and arm motion, yet whole-body trajectory optimization in cluttered, confined spaces remains difficult due to high-dimensional nonconvexity and the need for fast, accurate collision reasoning. Configuration Space Distance Fields (CDF) enable fixed-base manipulators to model collisions directly in configuration space via smooth, implicit distances. This representation holds strong potential to bypass the nonlinear configuration-to-workspace mapping while preserving accurate whole-body geometry and providing optimization-friendly collision costs. Yet, extending this capability to mobile manipulators is hindered by unbounded workspaces and tighter base–arm coupling. We lift this promise to mobile manipulation with Generalized Configuration Space Distance Fields (GCDF), extending CDF to robots with both translational and rotational joints in unbounded workspaces with tighter base–arm coupling.
We prove that GCDF preserves Euclidean-like local distance structure and accurately encodes whole-body geometry in configuration space, and develop a data generation and training pipeline that yields continuous neural GCDFs with accurate values and gradients, supporting efficient GPU-batched queries. Building on this representation, we develop a high-performance sequential convex optimization framework centered on GCDF-based collision reasoning. The solver scales to large numbers of implicit constraints through (i) online specification of neural constraints, (ii) sparsity-aware active-set detection with parallel batched evaluation across thousands of constraints, and (iii) incremental constraint management for rapid replanning under scene changes. Extensive randomized high-density benchmarks and real-robot experiments demonstrate consistently superior success rates, trajectory quality, and solve times compared to strong baselines, enabling fast, safe, and reliable whole-body planning with minimal initialization effort. Source code is available at the project website\footnote{\url{https://yulinli0.github.io/GCDF/}}.
\end{abstract}

\begin{figure}[t] 
    \setlength{\belowcaptionskip}{-20pt} 
    \centering
    \includegraphics[width=\linewidth, trim={8.5cm 0cm 9.5cm 0.2cm}, clip]{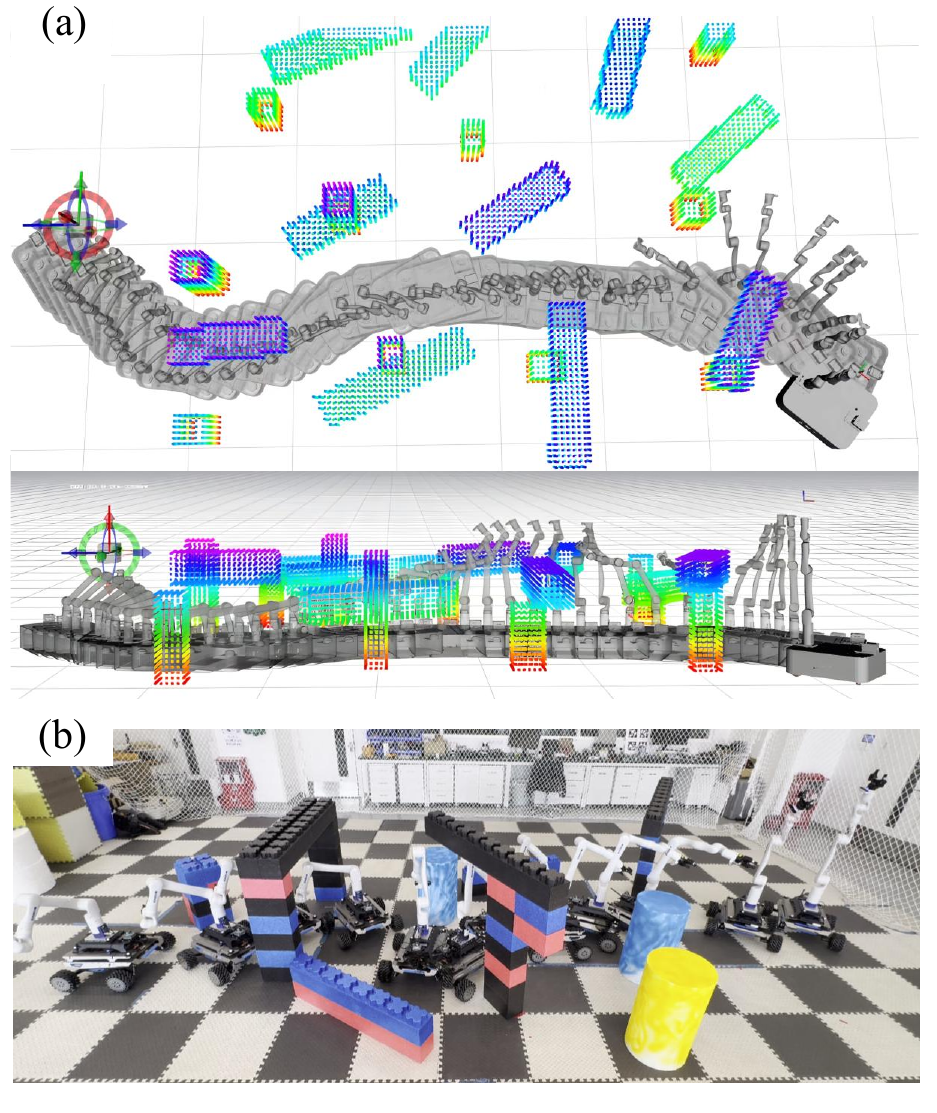}
    \caption{Trajectory generation in complex environments for mobile manipulators using the proposed numerical optimization algorithm with neural Generalized Configuration Space Distance Fields. Starting from a trivial initial guess (the mobile base linearly interpolated between start and goal while the manipulator joints are all zeros, i.e., kept upright), the solver generates collision-free trajectories from scratch, exhibiting smooth and agile maneuvers that leverage whole-body coordination for safe obstacle avoidance. (a) Simulation results are shown from front and top views. (b) Real-world experiments in a similarly cluttered setup.}
    \label{fig:overview}
\end{figure}
\section{Introduction}
\label{sec:introduction}
Mobile manipulators have emerged as essential tools for real-world robotic applications, combining mobility, dexterity, and workspace coverage to enable complex, contact-rich, and long-horizon tasks across both structured and unstructured environments~\cite{6385617,Pilania2018IJRR,Tzes2022ICRA,Michaux2025TRO}. To accomplish these tasks, mobile manipulators must leverage their multiple degrees of freedom to coordinate arm and base motion, continuously reconfiguring whole-body posture to adapt to complex, cluttered environments while maintaining \textit{agility} and ensuring \textit{safety}. However, this synergy of mobility and manipulation introduces a fundamental planning challenge: planners must make rapid, reliable decisions in high-dimensional configuration spaces while respecting collision constraints derived from accurate geometric models of both the robot and the scene, which are essential for achieving agile, non-conservative performance~\cite{howell2022ral-trajopt-optimizationbased-dynamics, 5152817,kalakrishnan2011irca-stomp,Pi2025TASE}. Despite recent advances in perception, whole-body planning, and optimization-based control for mobile manipulation~\cite{pankert2020perceptive,Wu2024ICRA,cHITTA2012RAM,Lillo2023TASE,gIFTTHALER2017ICRA}, fast and reliable trajectory generation in cluttered, confined remains an open challenge.

\textbf{Background.} 
Traditional hierarchical ``move-and-act'' schemes~\cite{8967733,1570432} decompose mobile manipulation into separate planning subproblems for the base and the arm, offering simplicity but constraining coordinated whole-body behaviors. This decoupling neglects the interdependence between base and manipulator, precluding whole-body reconfiguration to expand reachability in confined environments and leading to conservative or infeasible behavior~\cite{4543365,10472786,xia2021relmogenleveragingmotiongeneration}. To overcome these limitations, recent work has shifted toward whole-body planning, typically grouped into three families: sampling-based methods, reinforcement learning (RL), and optimization-based planning.

Sampling-based methods such as PRM/RRT and their variants~\cite{kavraki1996tra-prm,lavalle1998rr-rrt,rrtx,orthey2023sampling} explore feasible connectivity through random sampling and collision checking, offering theoretical guarantees and broad applicability. Recent works extend these to mobile manipulation by unifying base and arm in high-dimensional configuration space~\cite{10026619,9340782}. However, in cluttered environments with narrow passages requiring tight base-arm coordination, uniform or heuristic sampling becomes computationally burdensome and prone to connectivity failures. A couple of frameworks mitigate this by predefining constrained sampling regions or heuristic guidance~\cite{Wu2024ICRA,10472786}, but such strategies reintroduce scenario-dependent tuning and reduce planner generality. RL-based approaches learn control policies end-to-end, directly from state or sensor inputs to actions, typically by interacting with a simulator~\cite{10802201,Hu-RSS-23,10855801,kindle2020whole}. Such methods typically demand substantial simulation training for each new scenario, often requiring millions of timesteps and extensive hyperparameter tuning to achieve strong performance~\cite{10167547}. While effective in structured settings, the resulting policies often generalize poorly to novel layouts and dynamics and are sensitive to distribution shift and partial observability. Moreover, discrepancies in perception, dynamics, and actuation create a sim-to-real gap that can significantly degrade real-world performance~\cite{bi2025interactivenavigationleggedmanipulators}, necessitating careful engineering for reliable transfer.

In contrast, optimization-based methods offer a compelling alternative and serve as the focus of our work. By formulating whole-body mobile manipulation as a constrained optimization problem, they seek feasible, collision-free trajectories while optimizing task-relevant objectives such as smoothness, energy efficiency, or execution time~\cite{10610753,9341134}.
A key merit of this paradigm is its robustness and generalizability: once the environment is represented in a unified geometric~\cite{chunxin2025RAL-localreactive} or distance-field form~\cite{9981318}, the same optimization framework can be applied across diverse scenarios without scenario-specific retraining, enabling rapid adaptation to novel layouts and dynamic changes. Moreover, the kinematic and actuation redundancy inherent in mobile manipulators naturally lends itself to optimization techniques, allowing the system to exploit its full configuration space to generate agile, safe maneuvers that reach targets while simultaneously satisfying task-space collision avoidance and other operational constraints.

\textbf{Challenges. }Unlike fixed-base manipulators, mobile manipulators operate in unbounded workspaces and must navigate complex, cluttered obstacle layouts. Current state-of-the-art methods adopt a two-stage pipeline with front-end path search followed by back-end trajectory optimization~\cite{Wu2024ICRA,xu2025topayefficienttrajectoryplanning}. In practice, the front end carries the primary burden through elaborate, multi-layer search strategies aimed at reducing sampling difficulty over the combined translational and rotational motion space. A common approach decouples the problem by first planning a base trajectory and then searching for arm configurations along this fixed base path. The back-end optimizer is consequently reduced to a post-processing module that smooths and slightly refines the front-end solution locally, often converging within only a few iterations.

Despite its computational efficiency, this paradigm suffers from fundamental limitations in both stages. On the front-end side, decoupling the base and arm search undermines holistic whole-body coordination. In complex environments where simultaneous coordination is required amid intricate obstacle layouts, the front-end search can fail; this would cause the entire pipeline to collapse since the back-end optimizer is constrained to local refinements near the front-end solution. Once the front end produces a poor-quality trajectory, the optimizer cannot recover, despite its potential exploration capability.
On the back-end side, these methods typically rely on the environment Euclidean signed distance field (ESDF) and approximate the robot with collision spheres. While this reduces collision constraints complexity, it creates a fundamental mismatch: collision costs are defined in workspace as the sum of ESDF penalties over spheres, whereas the position of each sphere is coupled through highly nonlinear forward kinematics. Moreover, a sparse set of spheres can inaccurately represent the true collision geometry, leading to overly conservative or unsafe clearance estimates. From the optimization perspective, the nonlinear coupling yields a strongly non-convex landscape; with imperfect initial trajectories, gradient-based solvers easily become trapped in poor local minima as workspace gradients induce misleading configuration-space updates.

Overall, the difficulty stems from three root causes: (\emph{i}) the nonlinear mapping between configuration-space variables and workspace-based collision constraints; (\emph{ii}) the inherent computational intractability of enforcing whole-body collision avoidance with precise link geometries~\cite{li2024collisionfreetrajectoryoptimizationcluttered,10160716}, which necessitates simplified proxies (e.g., spheres or point sets~\cite{9561821,11119564}) that compromise geometric fidelity; (\emph{iii}) the unbounded workspace and intricate obstacle layouts, which make the optimizer highly sensitive to initialization yet render high-quality initial guesses non-trivial to obtain. It leaves an interesting question:   
\emph{Can we enforce collision avoidance directly in configuration space to circumvent the intricate nonlinear mapping between configuration variables and workspace constraints while preserving accurate robot geometries?}


Recently, inspired by the robot-centric signed distance field (SDF)~\cite{Aude2023RAL,Liu2022IROS,Li2024ICRA}, which measures the Euclidean distance between a robot surface and environment points, the concept of the Configuration Space Distance Field (CDF) has been proposed to measure the minimum distance from a current robot configuration to environment points using configuration metrics for robot manipulators~\cite{li2024RSS-CDF}. Specifically, for an obstacle point $\bm p$, CDF measures the minimum \textit{configuration space distance} between the current configuration and the zero-level configuration set, where the latter comprises all configurations that result in contact between any point on the robot's surface and $\bm p$. Thus, CDF naturally incorporates whole-body geometries without conservative approximation, while simultaneously embedding the inverse kinematic problem within its formulation to maintain Euclidean distance properties in the high-dimensional configuration space despite the highly nonlinear propagation along kinematic chains.

Crucially, we recognize that these properties make CDF particularly well-suited for gradient-based trajectory optimization. The constraint Jacobian matrices derived from CDF provide accurate local models that facilitate faster convergence in iterative solvers~\cite{jordana2023stagewise,sqp-method,doi:10.1177/0278364914528132}, while their continuous neural representation enables efficient batched queries of collision constraints and their derivatives across trajectory waypoints. This naturally motivates us to investigate how CDF can be effectively integrated into trajectory optimization frameworks to address the computational challenges outlined above.

\textbf{Contributions. }
In this work, we address the following research problems: (\emph{i}) How to generalize CDF to mobile manipulators with additional translational DoFs in an unbounded workspace, while preserving its merits for purely rotational joints. (\emph{ii}) How to learn implicit neural CDF representations for mobile manipulators that provide accurate values and gradients, while enabling fast, parallel online queries for numerical optimization. (\emph{iii}) How to design a high-performance numerical optimization algorithm with CDF-based constraints that enforces collision avoidance directly in configuration space, enabling rapid generation of agile, collision-free trajectories in complex obstacle layouts, even from naive initial guesses.

Specifically, we first propose the concept of Generalized CDF (GCDF) for robots with both translational and rotational joints. We show that, from both theoretical and experimental perspectives, the properties of the original CDF can be extended to GCDF with minimal modification. Unlike fixed-base manipulators, the introduction of translational dimensions leads to exponential growth in data requirements, and the quality and coverage of data on the zero-level set significantly impact both training effectiveness and subsequent optimization performance. To solve these, we carefully develop a data collection and training pipeline based on the established theories to obtain neural CDF functions for mobile manipulators, yielding continuous and accurate implicit representations that support efficient GPU-accelerated queries and are well-suited for integration with numerical optimization.

Second, we develop and open-source a high-performance C++ sequential convex optimization solver for large-scale trajectory optimization with implicit GCDF constraints. The solver employs an $\ell_1$-penalty formulation and iteratively solves QP subproblems using first-order constraint information within a local trust region, together with globalization mechanisms that enable recovery from initial guesses far from the optimum. To support neural implicit collision reasoning, we extend the framework to (a) transform neural network outputs as implicit constraints; (b) batch-evaluate constraint values and gradients in parallel and exploit the sparsity pattern of active collision constraints at each iteration, substantially reducing both problem assembly and solve time under thousands of constraints; and (c) allow online injection and removal of GCDF constraints without problem reconstruction, which is critical for rapid replanning as robot--obstacle interactions change. As previewed in Fig.~\ref{fig:overview}, the resulting solver optimizes directly in configuration space and can recover from a trivial initial guess to produce agile, collision-free motions in highly cluttered scenes.

We emphasize that obstacle density and distribution complexity are critical factors for validating planning performance, as both the computational burden (due to the number of constraints) and the difficulty of escaping inferior local minima increase dramatically with environmental complexity. Finally, to rigorously validate the superiority of our framework, we design multiple randomized high-density obstacle scenarios and conduct extensive comparisons with benchmark methods in these challenging environments. Results demonstrate that the developed numerical algorithm, with GCDF constraints, can rapidly compute safe, robust, and non-trivial whole-body motions even in highly cluttered spaces. We further deploy the algorithm on a real mobile manipulator system in complex real-world scenarios, demonstrating its robustness and effectiveness in practical applications.     

To summarize, our contributions are:
    
    
    
\begin{itemize}
    \item \textbf{Generalized CDF.} We introduce GCDF for robots with both translational and rotational joints, and show, both theoretically and empirically, that core CDF properties carry over with minimal modification.
    \item \textbf{Neural GCDF.} We develop a scalable data generation and training pipeline that avoids data explosion in unbounded workspaces by reconstructing high-quality coverage of the zero-level set through translation-equivariant aggregation of grid-based subsets, yielding neural GCDFs with accurate values/gradients and fast GPU-parallel queries.
        \item \textbf{Solver Implementation.} We release an open-source high-performance C++ solver that enforces collision avoidance directly in configuration space via neural implicit GCDF constraints, with online constraint specification, batched GPU evaluation, and sparsity-aware handling of large-scale constraints.
    \item \textbf{Performance.} Our method achieves consistently superior performance for rapid, safe, and reliable collision-free whole-body planning, outperforming strong baselines in randomized high-density clutter and validating on real hardware.
\end{itemize}



\textbf{Paper organization.} We review related work in Section \ref{sec:relatedworks}. We present the theory and training pipeline for GCDF in Section~\ref{sec:method:cdf-mobile-manipulator}, followed by its integration with the numerical algorithm for trajectory optimization in Section~\ref{sec:method:trajectory-optimization}. Experimental results are demonstrated in Section~\ref{sec:exp}. We conclude with future directions in Section~\ref{sec:conclusion}.
\section{Related Works}
\label{sec:relatedworks}
\subsection{Collision Avoidance}
\label{sec:relatedworks:collision_avoidance_considerations}
Approaches to collision avoidance in robot motion planning can be broadly categorized into two distinct paradigms, defined by their underlying representation of safety. The first paradigm focuses on explicitly constructing a representation of the collision-free space, and the second paradigm operates by modeling collision pairs and enforcing safety margins.

\subsubsection{Motion Planning in Collision-Free Space}
This category treats motion planning as a two-stage process: first constructing a structural representation of the collision-free space, then computing feasible trajectories within that structure. Classical approaches, including cell decomposition~\cite{latombe2012robot,SCHWARTZ1983298}, visibility graphs~\cite{asano1985visibility,lozano1979algorithm}, and Voronoi diagrams~\cite{4539723}, systematically partition or reduce the configuration space into discrete geometric structures. While effective in low-dimensional settings, these deterministic methods become computationally intractable in high-dimensional spaces. Sampling-based planners~\cite{kavraki1996probabilistic,lavalle1998rapidly,rrtx} address this by probabilistically capturing connectivity, achieving widespread adoption in robot manipulations~\cite{6377468,coleman2014reducingbarrierentrycomplex,5152399}.
More recently, modern optimization-based approaches favor decomposing free space into sequences of convex regions~\cite{deits2015computing,wang2025fast,7839930,toumieh2022voxel,10802782}. These methods grow convex polytopes in collision-free space and then constrain robot motion by enforcing containment relationships between robot geometry and these safe regions~\cite{10189357,9268421}. Depending on the geometric representation, such as points~\cite{7839930,tordesillas2021faster}, polytopes~\cite{dai2023certified,10950074}, or semialgebraic sets~\cite{li2024collisionfreetrajectoryoptimizationcluttered}, different containment constraints can be formulated within the optimization framework. This paradigm has been extended to manipulators and mobile manipulators by allocating separate convex regions to each link~\cite{10950074,chunxin2025RAL-localreactive,9561821}, enabling coordinated whole-body planning.
\subsubsection{Safety Margins Between Collision Pairs}
The second category focuses on explicitly modeling the geometric relationship between robot and obstacles to enforce positive distance margins. Compared to convex decomposition methods, which face combinatorial challenges in selecting which free region the robot should occupy at each time step~\cite{tordesillas2021faster,deits2015icra-mixedintegerprogramming-uav} and ensuring safe transitions between regions~\cite{li2024collisionfreetrajectoryoptimizationcluttered} (as overlapping quality between adjacent regions is often difficult to guarantee), distance-based methods provide a more direct and complete solution: \textit{safety can be absolutely guaranteed by enforcing minimum distance thresholds between the robot and all obstacles.}
The most straightforward approach explicitly constructs distance functions between robot-obstacle collision pairs and enforces safety thresholds~\cite{5152817,kalakrishnan2011irca-stomp,mellinger2011minimum,ma2022alternating}. However, distance functions for general geometries in complex configuration spaces are inherently difficult to obtain, typically leading to oversimplified collision representations such as points or spheres. Some indirect methods have been developed to handle more general shapes for single rigid bodies: control barrier functions (CBFs) based on the dual formulation of minimum distance between polytopes~\cite{9812334}, and bi-level optimization that enforces scaling factors on geometric primitives to ensure collision-free configurations~\cite{10160716}. Another approach computes separating hyperplanes between the robot and obstacle convex hulls~\cite{9490372,nair2022collision}. In principle, these methods jointly optimize both the trajectory variables and the hyperplane parameters at each time step, ensuring that a separating plane can be found between the robot and obstacles throughout the motion.

However, for multi-link robotic systems such as manipulators and mobile manipulators, these direct approaches remain difficult to generalize due to kinematic coupling between links and the high dimensionality of the configuration space. As an alternative, distance fields have been widely adopted. The principle behind this approach is to build an implicit representation that encodes the minimum distance between environmental points and their nearest robot surface points~\cite{Ante-24RAL-Online-Learing-SDF,Li2024ICRA,JJPARK-19CVPR-LearningSDF} or nearest obstacle points~\cite{Helen-17IROS-VOXBLOX,Zhong-23ICRA-ShineMapping,Ortiz:etal:iSDF2022}. This implicit structure allows querying minimum distance values and their gradients with respect to robot configurations, which can be utilized to detect collisions and generate repulsive forces that push the robot away from obstacles~\cite{li2024RSS-CDF,bhardwaj2022storm}. These methods have great potential for application to robotic systems with complex configurations and are the focus of this paper. We provide a comprehensive review of distance fields and their evolution toward configuration-space representations in Section~\ref{sec:relatedworks:distance_fields}.

\subsection{Distance Fields and their Applications}
\label{sec:relatedworks:distance_fields}
Distance fields, and in particular SDFs, encode at every point the signed distance to the closest surface. This scalar field captures three essential aspects of safety in a unified way: the zero level set indicates whether a point is in collision, the distance value measures how much clearance it has, and the gradient specifies in which direction clearance increases. In robotics, this makes SDFs a natural choice for linking perception and motion generation. Voxel-based TSDF/ESDF maps support dense 3D reconstruction and local planning for mobile robots and MAVs~\cite{Helen-17IROS-VOXBLOX,9981318}, while continuous and neural SDFs~\cite{Ortiz:etal:iSDF2022,vasilopoulos2024hio,Ante-24RAL-Online-Learing-SDF} provide compact scene representations with efficient distance and gradient queries. These environment-centric distance fields are routinely used as smooth collision costs or repulsive potentials in trajectory optimization~\cite{5152817} and MPC-based planners~\cite{koptev2024reactive}, giving motion-generation algorithms a continuous ``safety landscape'' instead of a binary collision oracle.

More recently, there has been a shift from environment-centric to robot-centric distance fields, in which the \emph{robot} itself is represented as a distance field. This removes the need to update a global scene SDF online as the environment changes, enabling more efficient distance queries and improving responsiveness in dynamic scenarios~\cite{Aude2023RAL,Li2024ICRA}. Moreover, a robot SDF provides derivatives with respect to joint angles, encoding both proximity to collision and directions in joint space that increase clearance, which aligns well with gradient-based motion optimization. A central challenge is how to couple the distance representation with the robot joint configuration. Existing methods address this in two principal ways: kinematic-chain-based models explicitly propagate fixed link-wise SDFs through forward kinematics, whereas end-to-end learning approaches embed geometry and kinematics into a single joint-dependent distance function. Methods such as~\cite{Li2024ICRA,liu2023collision,zhu2025efficient,chen2025implicit} explicitly encode link geometries and compose them along the known kinematic chain, preserving geometric fidelity and yielding well-structured derivatives from task space to joint space, which is advantageous in scenarios with tight clearances and detailed contact reasoning. In contrast, end-to-end neural approaches learn a single implicit model that maps joint configurations and workspace points directly to distances and gradients~\cite{Aude2023RAL,Liu2022IROS}. These models trade some geometric accuracy for highly efficient, uniform inference, making them suitable for high-frequency reactive control and large-scale motion generation where many distance queries must be evaluated online. In both cases, representing the robot as a distance field yields a unified, differentiable representation for reasoning about self-collision, contact, and clearance.

While most distance-field methods operate in task space, there is increasing interest in lifting distance reasoning directly to configuration space. CDF~\cite{li2024RSS-CDF} defines a scalar field over joint configurations, where the value at each configuration equals the minimal joint motion required to reach the collision manifold, with its gradient providing the corresponding avoidance direction. In contrast to task-space SDFs, CDF satisfies an eikonal property in configuration space: its level sets are uniformly spaced, and its gradient has unit norm and consistently points away from the collision set. This eliminates distortions or vanishing gradients induced by the nonlinear mapping between task space and joint space, and enables one-step gradient projection for inverse kinematics. Consequently, motion generation schemes built around distance fields can be applied directly in joint space by replacing a workspace SDF with a joint-space CDF. Building on this idea, subsequent work has leveraged CDF for motion planning and control by integrating it with barrier-function formulations~\cite{long2025neural}, gradient-based trajectory optimization~\cite{chi2024safe}, and model predictive path integral control~\cite{li2025one}, highlighting CDF as an effective representation for efficient and safe motion generation in joint space.
However, these approaches focus primarily on fixed-base manipulators, and CDF for mobile manipulators where the configuration space couples Euclidean base motion with joint orientations, have not yet been systematically explored. Moreover, efficient and robust numerical solvers capable of handling such structured, implicitly defined constraints are still lacking, limiting the practicality of existing methods.


\section{Neural Configuration Space Distance Field for Mobile Manipulators}
\label{sec:method:cdf-mobile-manipulator}
\subsection{Preliminaries}
Let us first recall the original definition of CDF.

\subsubsection{Signed Distance Field}
Before introducing the definition of CDF, we first examine SDF, a widely adopted concept in collision checking and avoidance tasks, as it helps illuminate the advantages of CDF.

SDF was originally defined in task space as the distance from a query point $\boldsymbol{p}$ to the boundary of a set $\Omega$, with positive and negative signs indicating whether the point lies outside or inside $\Omega$, respectively. As illustrated in Fig.~\ref{fig:sdf_mobile_manipulator}-\emph{left}
, this concept has been extended to represent the distance between a workspace point $\boldsymbol{p}$ and the robot surface, denoted as $\partial r$, at a given configuration $\boldsymbol{q}$:
\begin{equation}
\label{eq:sdf-cspace}
f_s(\bm{p}, \bm{q}) = \pm \min_{\bm{p}' \in \partial r(\bm{q})} \|\bm{p'} - \bm p \|.
\end{equation}
While the signed distance $f_s$ maintains Euclidean properties in workspace (or task space), the forward kinematics embedded in the manifold constraint $\bm {p'}\in \partial r(\bm q)$ introduces significant nonlinearity in configuration space, where the optimization is conducted. 
\begin{figure}[htbp] 
    \centering
    \includegraphics[width=0.5\textwidth, trim={0cm 0cm 0cm 0cm}, clip]{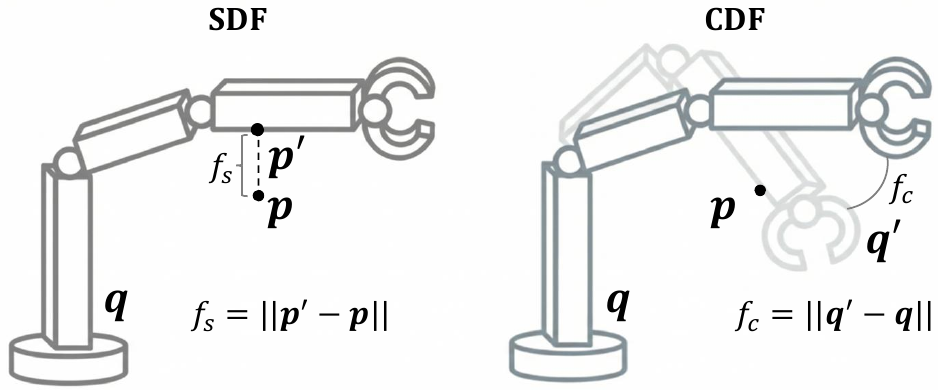}
    \vspace{.8mm}
    \caption{Illustration of SDF (\emph{left}) and CDF (\emph{right}) for a query point $\bm p$ at configuration $\bm q$. SDF measures the task-space distance from $\bm p$ to the nearest point $\bm p'$ on the robot surface, as defined in~\eqref{eq:sdf-cspace}. In contrast, CDF measures the configuration-space (angular) distance from $\bm q$ to the nearest contact configuration $\bm q'$ on the zero-level set induced by $\bm p$, as defined in~\eqref{eq:cdf}.}
    \label{fig:sdf_mobile_manipulator}
\end{figure}

\subsubsection{Configuration Space Distance Field} 
We now present the concept of CDF. In contrast to the minimum distance in task space as defined in (\ref{eq:sdf-cspace}), CDF is defined as the minimum distance in configuration space. Specifically, for an articulated robotic system with pure rotational joints, this represents the minimum distance in radius~\cite{li2024RSS-CDF}, as illustrated in Fig.~\ref{fig:sdf_mobile_manipulator}-\emph{right}.

First, we define the zero-level configuration set for a given workspace point $\bm p$ as:
\begin{equation}
\label{eq:zero-level-set}
\calZ(\bm p) = \{\bm q \mid f_s(\bm p,\bm q) = 0\},
\end{equation}
where, based on the SDF definition in (\ref{eq:sdf-cspace}), $\calZ$ represents the set of all configurations where $\bm p$ lies on the robot surface. The CDF is then formally defined as:
\begin{equation}
\label{eq:cdf}
f_c(\bm p, \bm q) = \min_{\bm{q'} \in \calZ(\bm p)}\|\bm{q'} - \bm q\|.
\end{equation}
For a given point $\bm p$, CDF measures the minimum radial distance from the current configuration to any configuration that results in contact with $\bm p$. 
\begin{remark}
    CDF effectively bridges task space and configuration space while preserving several advantageous properties: it maintains the implicit structure and Boolean operations of SDF, and more significantly, it preserves the uniform gradient property with Euclidean distance in configuration space. These characteristics collectively make CDF an ideal candidate for integration with gradient-based numerical optimization algorithms, enabling simultaneous safety guarantees and task achievement. Moreover, the efficient data collection and network training proposed in~\cite{li2024RSS-CDF} provides a practical pipeline for neural network-encoded CDF function that enables fast batch-parallel computation of both distance values and gradients during online operation. 
\end{remark}

\subsection{Configuration Space Distance Field for Mobile Manipulator}
\subsubsection{Generalized CDF}
Originally, when training the CDF function, only rotational joints are considered since the planar translation motion (i.e., $\bm q^t$) would make the workspace unbounded, and there are inherently different scales between translational and rotational movements.

Theoretically, the shift caused by base translation can be mitigated by transforming obstacle points to the mobile base frame, and the gradients with respect to translational degrees of freedom can be obtained through the chain rule by utilizing the gradients of obstacle point positions. However, in practical training scenarios, fixing the mobile base significantly constrains the actual workspace and confines the reachable space. This leads to insufficient and non-general zero-level-set data, resulting in obstacle gradients that lack physical intuition. Therefore, in our training process, we redefine the CDF to incorporate translational degrees of freedom while introducing appropriate scaling factors to balance translational distances and rotational angles. Similar to (\ref{eq:cdf}), we introduce the definition of the generalized CDF with translational Dofs:
\begin{tcolorbox}[colback=gray!10, colframe=gray!80, boxrule=0.5pt, arc=2mm]
\begin{definition}[Generalized CDF]
\label{def:cdf_generalized}
For a robot configuration $\boldsymbol{q}\in\mathbb{R}^{n}$ comprising a translational component $\bm q^t$ and a rotational component $\bm q^r$, and an environmental point $\boldsymbol{p}\in\mathbb{R}^{3}$, the generalized configuration space distance field (GCDF) $f^g_c(\bm p, \bm q):\mathbb{R}^{n}\times\mathbb{R}^{3} \rightarrow \mathbb{R}$ is:
\begin{equation}
\label{eq:cdf-weighted}
f^g_c(\bm p, \bm q) = \pm \min_{\bm{q'} \in \calZ(\bm p)}\|\bm{q'} - \bm q\|_{\bm M},
\end{equation}
where $\calZ(\bm p)$ denotes the zero level set manifold of configurations satisfying the contact condition at point $\bm p$, and ${\bm M} \in \bm{S}_+^n$ is a positive semidefinite diagonal weighting matrix that induces the weighted norm $\|\cdot\|_{\bm M}$. The positive
and negative signs indicate whether $\bm p$ lies outside or
inside the robot surface $\Omega(\bm q)$.
\end{definition}
\end{tcolorbox}
\begin{remark}
The introduction of the weighting matrix ${\bm M}$ in this generalized formulation serves two crucial purposes. First, it controls the relative influence of translational and rotational distances on the CDF value. Second, it shapes the characteristics of the nearest configuration by influencing the gradient direction and consequently the robot's motion tendency in proximity to objects. 
\end{remark}

A key theoretical foundation of the original CDF is its satisfaction of the \textit{eikonal equation} in configuration space, which ensures unit gradient norms and enables single-step projections to the zero level set. This property theoretically guarantees closed-form whole-body inverse kinematic solutions without numerical iterations. To maintain this valuable property during training, gradient norm and projection error regularization terms are incorporated into the loss function. However, when extending CDF to the full configuration space with weighted norms, a critical issue emerges: following the gradient direction no longer guarantees reaching the zero level set surface, and the optimal step size becomes undetermined. This challenges both the theoretical properties and their corresponding training objectives. Following the spirit of the original CDF framework, we extend its theoretical properties to the weighted configuration space through the following theorem:
\begin{tcolorbox}[colback=gray!10, colframe=gray!80, boxrule=0.5pt, arc=2mm]
\begin{theorem}[Properties of GCDF]
\label{thm:gcdf_properties}
    For any configuration $\bm q$ and environmental points $\bm p$ where $f^g_c(\bm p, \bm q)$ is differentiable with respect to $\bm q$, the following properties hold:
    \begin{enumerate}
        \item (Weighted Eikonal Equation). The partial derivative of $f^g_c$ with respect to $\bm q$ satisfies the weighted eikonal equation:
        $$\|\nabla_{\bm q} f^g_c(\bm p, \bm q)\|_{{\bm M}^{-1}} = 1$$
        \item (Single-Step Projection). For any initial configuration $\bm q_0$, its closest configuration $\bm q_z \in \calZ(\bm p)$ can be reached through a single-step projection:
        $$
        \bm q_z = \bm q_0 - \lambda d,
        $$
        where $\lambda = f^g_c(\bm p, \bm q_0)$ denotes the step length and $d = {\bm M}^{-1}\nabla_{\bm q} f^g_c(\bm p, \bm q_0)$ defines the projection direction.
    \end{enumerate}
\end{theorem}
\begin{proof}
See~\prettyref{app:gcdf_properties}.
\end{proof}
\end{tcolorbox}
While Theorem~\ref{thm:gcdf_properties} provides theoretical guidance for training effective GCDF and enables finding the nearest configurations on the zero-level set through biased gradient projection, practical considerations necessitate the signed formulation in Definition~\ref{def:cdf_generalized}. Specifically, the unsigned nature of the original CDF definition in~(\ref{eq:cdf}) would cause numerical issues for gradient-based optimization when the robot is currently intersecting with obstacles.  
To illustrate this, consider the one-dimensional case shown in Fig.~\ref{fig:cdf_one_dimension}. If the configuration $\bm q$ to be optimized initially lies between two points on the zero-level set, gradient-based numerical methods theoretically cannot escape the local maximum at $\bm q_0$ (\ie, pushing the obstacle away from the robot). Therefore, the signed GCDF formulation resolves this by reversing the sign for penetrating configurations, ensuring gradients consistently point toward collision-free space.

\begin{figure}[htbp]
\setlength{\belowcaptionskip}{-10pt}
\centering
\includegraphics[width=0.4\textwidth, trim={0cm 0cm 0cm 0cm}, clip]{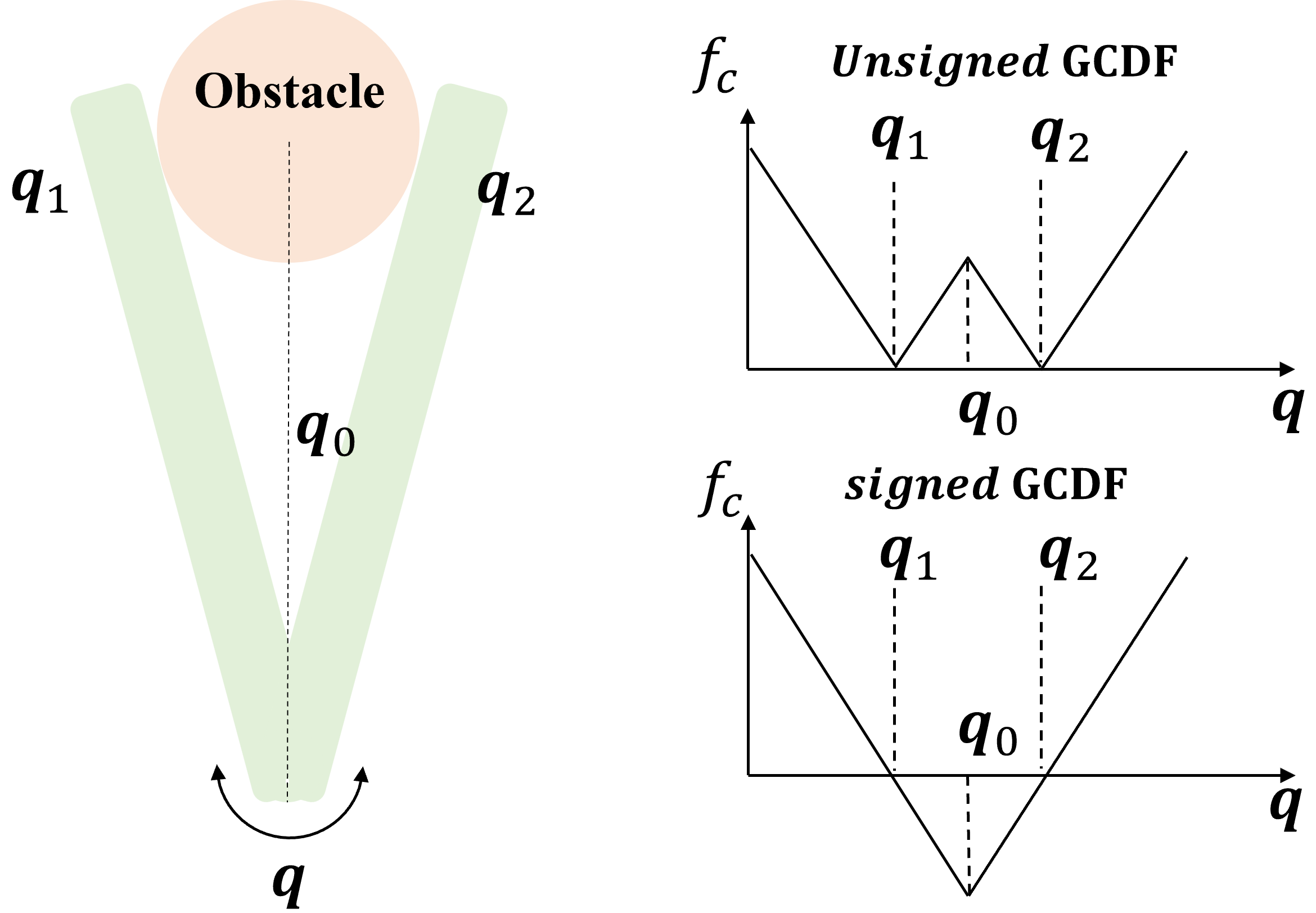}
\caption{Illustration of GCDF values for a 1-DoF planar arm. The obstacle induces two symmetric contact configurations on the zero-level set, denoted by $\bm q_1$ and $\bm q_2$. For the \emph{unsigned} GCDF, the distance function exhibits a spurious local maximum at $\bm q_0$, where the obstacle is equally distant from both arm sides, which can mislead gradient-based numerical solvers. Introducing the sign based on penetration removes this pathology, yielding smooth and consistent gradient directions toward the contact set.}

\label{fig:cdf_one_dimension}
\end{figure}

\subsection{Neural GCDF}
Having established the definition and properties of GCDF, we now present the training methodology for neural GCDF representations mapping from workspace-configuration pairs to their corresponding GCDF values, which ensure high-quality distance fields suitable for numerical optimization. Our pipeline generally follows the procedure proposed in~\cite{li2024RSS-CDF}:

\begin{enumerate}[label=(\roman*)]
    \item \textbf{Robot SDF Construction.} We obtain the SDF representation of the mobile manipulator, i.e., $f_s(\bm p, \bm q)$, using the method in~\cite{Li2024ICRA}, enabling efficient distance queries and gradient computation.
    
    \item \textbf{Zero-Level Set Approximation.} We collect a workspace point set $\mathcal{P} \subset \mathbb{R}^{3}$. For each $\bm p \in \mathcal{P}$, we approximate the zero-level set $\mathcal{Z}(\bm p)$ by performing gradient descent on the robot SDF, starting from configurations $\bm q$ sampled from a configuration set $\mathcal{Q} \subset \mathbb{R}^{n}$.
    
    \item \textbf{Dataset Preparation.} For each newly sampled workspace-configuration pair $(\bm p, \bm q)$, we compute the ground-truth GCDF value according to Definition~\ref{def:cdf_generalized} by finding the minimum distance from $\bm q$ to the approximated zero level set $\mathcal{Z}(\bm p)$.
    
    \item \textbf{Neural Network Training.} We design the network architecture and loss functions, then train the neural implicit function to regress GCDF values across the configuration space.
\end{enumerate}
However, extending this pipeline to mobile manipulators introduces several critical challenges that must be addressed. 
The first challenge stems from the exponential growth of zero-level sets. Unlike fixed-base manipulators with bounded workspaces, mobile manipulators operate in unbounded translational spaces. For a given workspace point $\bm p$, the corresponding zero-level set $\mathcal{Z}(\bm p)$ contains exponentially more configurations due to the infinite reachability from different base positions. This leads to two major issues: (i) the computational cost of approximating $\mathcal{Z}(\bm p)$ can take days to complete, and (ii) gradient descent from slightly different base positions often converges to nearly identical arm configurations, resulting in massive redundancy. Consequently, the standard approach of sampling a fixed number of configurations becomes insufficient to construct informative zero-level sets, leading to significant bias when querying GCDF values during training.

The second challenge involves balancing model complexity and efficiency. To ensure compatibility with downstream numerical optimization, we must carefully balance inference efficiency against model accuracy. This requires joint consideration of network architecture complexity and data density. Indeed, overly complex networks or dense sampling dramatically increase training costs, while insufficient capacity or sparse data compromise the smoothness and accuracy of the learned GCDF, degrading optimization performance.

The third challenge concerns loss function design. The standard loss functions designed for fixed-base manipulators fail to capture the unique properties of mobile manipulator GCDFs. We require modified loss formulations that account for the increased dimensionality and ensure the learned distance field exhibits the smoothness and gradient consistency necessary for effective collision avoidance optimization.
We now present our solutions to these challenges.
\begin{figure}[t]
\setlength{\belowcaptionskip}{-10pt}
\centering
\includegraphics[width=0.5\textwidth, trim={0cm 0cm 0cm 0cm}, clip]{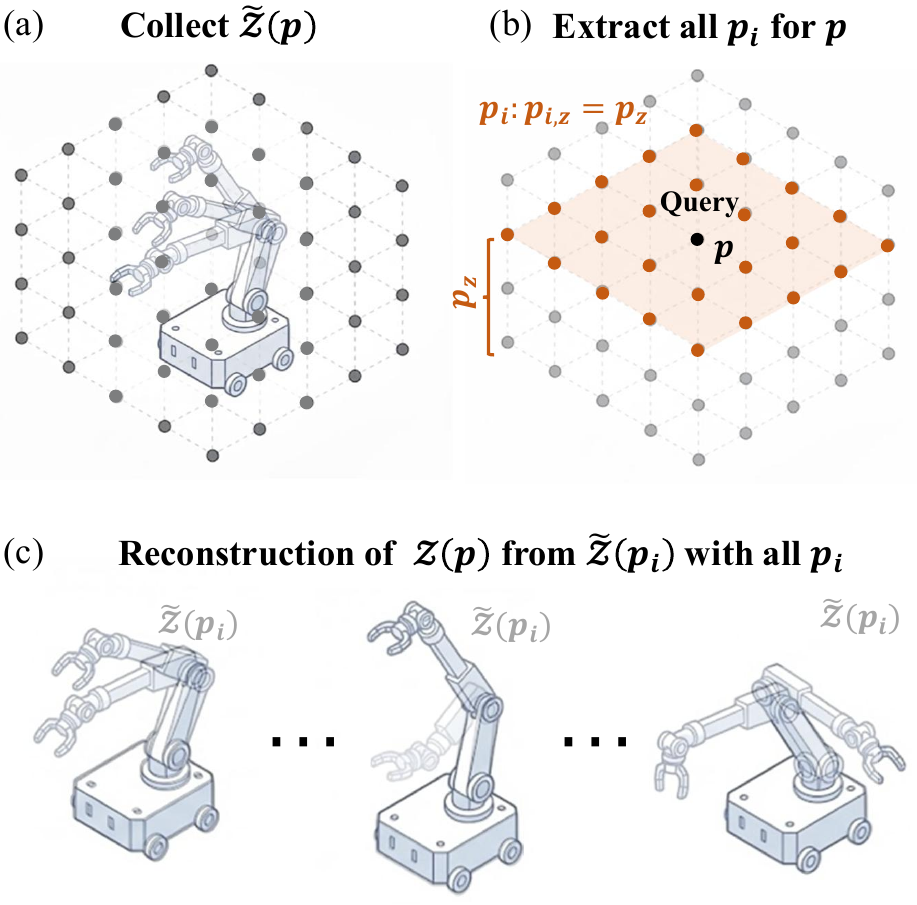}
\caption{Construction of zero-level set for mobile manipulators. 
(a) For each workspace point $\bm p$ on the grid centered at the robot origin, we construct a representative subset $\tilde{\mathcal{Z}}(\bm p)\subset\mathcal{Z}(\bm p)$ by fixing the base translation and solving for contact configurations over the arm joints and base rotations. 
(b) For a new query point $\bm p$ with height $p_z$, we gather all grid points $\bm p_i$ on the same horizontal slice ($\bm p_{i,z}=\bm p_z$) in a neighborhood around $\bm p$.
(c) We reconstruct the full zero-level set $\mathcal{Z}(\bm p)$ by taking the union of the precomputed subsets on that slice and compensating for their horizontal offsets, effectively translating the base while reusing the same arm contact configurations.}
\label{fig:neural-gcdf-training}
\end{figure}

\textbf{Data Preparation. }
Our data preparation process consists of two phases: (1) offline computation of zero-level sets, and (2) online sampling of representative workspace-configuration pairs with their corresponding GCDF values.
As we noted earlier, when constructing the zero-level sets of $\bm p$, considering all possible base positions $(q^t_x, q^t_y)$ would lead to data explosion and information redundancy, as the reachable workspace becomes unbounded when the base can translate freely. To mitigate this issue, we fix the base position at the origin ($q^t_x = q^t_y = 0$) during the offline phase, computing only a representative subset of the true zero-level set:
\begin{equation}
\tilde{\mathcal{Z}}(\bm p) = \{\bm q \in \mathcal{Z}(\bm p) \mid q^t_x = 0, q^t_y = 0\}
\label{eq:zero_level_subset}
\end{equation}
Specifically, as shown in Fig.~\ref{fig:neural-gcdf-training}(a), we first discretize the workspace around the robot center into a $T \times T \times T$ volumetric grid. For each grid point $\bm p$, we compute a dense subset $\tilde{\mathcal{Z}}(\bm p)$ efficiently by concurrently solving the following unconstrained optimization problem from a batch of randomly sampled initial arm configurations of $N$ initial arm and base rotational configurations using a quasi-Newton method~\cite{nocedal1999numerical}:
\begin{equation}
\min_{\bm q}  f_s^2(\bm p, \bm q), 
\label{eq:zero_level_opt}
\end{equation}
Note that we freeze the translational elements in $\bm q$ at the origin during optimization, focusing solely on finding optimal arm and base rotations. This subset $\tilde{\mathcal{Z}}(\bm p)$ addresses the computational challenge and avoids the data explosion problem, but it inevitably loses information about base mobility.

To recover the missing information during the online phase, we leverage a key insight about mobile manipulators: \textit{the arm’s relative reachability with respect to obstacles is translation-equivariant in the horizontal plane.}. In other words, if a manipulator can reach a point $\bm p$ from base position $(0, 0)$ with joint configuration $\bm q^r$, then it can reach point $\bm p'$ from base position $(\Delta x, \Delta y)$ with the same joint configuration $\bm q^r$, where $\bm p' = \bm p + [\Delta x, \Delta y, 0]^\top$.

Based on this observation, for a new query point $\bm p = [p_x, p_y, p_z]^\top$, we can reconstruct its complete zero-level set $\mathcal{Z}(\bm p)$ by aggregating the precomputed subsets from all grid points $\bm p_i$ centered around $\bm p$ with the same height, as illustrated in Fig.~\ref{fig:neural-gcdf-training}(b)-Fig.~\ref{fig:neural-gcdf-training}(c), while compensating for the horizontal displacement:
\begin{equation}
\mathcal{Z}(\bm p) = \bigcup_{i: \bm p_{i,z} = \bm p_z} \left\lbrace [\bm p_{i,x}, \bm p_{i,y}, {\bm q^r}^\top]^\top \mid \bm q^r \in \tilde{\mathcal{Z}}(\bm p_i) \right\rbrace.
\label{eq:zero_level_reconstruct}
\end{equation}
Intuitively, this is equivalent to virtually translating the robot base to position $\bm p_i$, treating it as the new origin, and then combining the configurations from $\tilde{\mathcal{Z}}(\bm p_i)$ to reach the target point.  This concatenation of zero-level sets from the $T \times T$ grid points $\bm p_i$, with appropriate base position adjustments, recovers the rich information about base mobility that was sacrificed in the offline phase.

For each workspace-configuration pair $(\bm p, \bm q)$, the ground truth GCDF value is computed as:
\begin{equation}
f^g_c(\bm p, \bm q) = \text{sgn}(f_s)  \min_{\bm q_i \in \mathcal{Z}(\bm p)} \|\bm q - \bm q_i\|_{\bm M},
\label{eq:cdf-weighted}
\end{equation}
where $\bm M$ is the weight matrices defined in Definition~\ref{def:cdf_generalized}, and $\text{sgn}(f_s)$ indicates the collision status determined by the sign of the SDF value $f_s(\bm p, \bm q)$. This formulation accounts for both the joint-space distance and the horizontal displacement cost, providing a comprehensive measure of configuration proximity. 

\textbf{Network Architecture and Loss Design. }
The Neural GCDF employs a 7-layer MLP architecture. The MLP is trained using the concatenation of $\bm{p}$ and $\bm{q}$ as input, we randomly sample $b_1$ grid points in the workspace and $b_2$ configurations for each point, resulting in an input tensor of size $\mathbb{R}^{(b_1 b_2) \times (3+n)}$. The network outputs the corresponding GCDF value for each input pair, resulting in an output of size $\mathbb{R}^{b_1 \times b_2}$.
\begin{table}[t]
\caption{\textbf{Loss function design for training neural GCDF.}}
\centering
\begin{tabular}{@{}l >{\centering\arraybackslash}p{6.5cm} p{0cm}@{}}
\toprule
\textbf{Loss } & \textbf{Definition} &  \\
\midrule
\( \mathcal{L}_{\text{dist}} \)  & 
\(
\displaystyle
\frac{1}{b_1 b_2} \sum_{i=1}^{b_1} \sum_{j=1}^{b_2} \left( \hat{f}_c^g(p_i, q_j) - f_c^g(p_i, q_j) \right)^2
\)
& \\
\addlinespace
\( \mathcal{L}_{\text{grad}} \)   & 
\(
\displaystyle
\frac{1}{b_1 b_2} \sum_{i=1}^{b_1} \sum_{j=1}^{b_2} 
\left( 1 - 
\frac{
    \nabla_{\bm{q}} \hat{f}_c^g(\bm{p}_i, \bm{q}_j)^T \cdot \nabla_{\bm{q}} f_c^g(\bm{p}_i, \bm{q}_j)
}{
    \| \nabla_{\bm{q}} \hat{f}_c^g(\bm{p}_i, \bm{q}_j) \|
    \| \nabla_{\bm{q}} f_c^g(\bm{p}_i, \bm{q}_j) \|
} \right)^2
\)
& \\
\addlinespace
\( \mathcal{L}_{\text{eikonal}} \)   & 
\(
\displaystyle
\frac{1}{b_1 b_2} \sum_{i=1}^{b_1} \sum_{j=1}^{b_2} \left( \| \nabla_{\bm q} f^g_c(\bm p, \bm q)\|_{{\bm M}^{-1}} - 1 \right)^2
\)
&  \\
\addlinespace
\( \mathcal{L}_{\text{tension}} \) & 
\(
\displaystyle
\frac{1}{b_1 b_2} \sum_{i=1}^{b_1} \sum_{j=1}^{b_2} \| \nabla_q \hat{f}_c^g(p_i, q_j) \|^2
\)
&  \\
\bottomrule
\end{tabular}
\label{tab:loss function}
\end{table}
The design of our neural GCDF loss function is inspired by the original CDF framework~\cite{li2024RSS-CDF} and tailored to the theoretical properties of GCDF established in Theorem~\ref{thm:gcdf_properties}. Specifically, we incorporate the distance loss $\mathcal{L}_{\text{dist}}$ for distribution fitting, the gradient loss $\mathcal{L}_{\text{grad}}$ for gradient direction regularization, the tension loss $\mathcal{L}_{\text{tension}}$ for smoothness enhancement, and the eikonal loss $\mathcal{L}_{\text{eikonal}}$ to enforce the weighted eikonal property. These loss terms collectively ensure that the trained implicit neural GCDF possesses valid field values and gradient magnitudes/directions that conform to the theoretical guarantees in Theorem~\ref{thm:gcdf_properties}, enabling nearest configurations on the zero-level set to be reached via single-step gradient projection.

Detailed definitions of these loss components are provided in Table~\ref{tab:loss function}, with the variables denoted by a hat ( $\hat{}$ ) representing the ground truth values. The total loss is formulated as a weighted sum of the four components:
\begin{equation}
\mathcal{L}_{\text{total}} = 
\lambda_1 \mathcal{L}_{\text{dist}} +
\lambda_2 \mathcal{L}_{\text{grad}} +
\lambda_3 \mathcal{L}_{\text{eikonal}} +
\lambda_4 \mathcal{L}_{\text{tension}}.
\label{eq:total_lo}
\end{equation}

\section{Fast and Safe Trajectory Optimization}
\label{sec:method:trajectory-optimization}
Building upon the neural GCDF representation trained offline (Section~\ref{sec:method:cdf-mobile-manipulator}), which provides dense value and gradient information of GCDF through a PyTorch model, this section presents our trajectory optimization framework for mobile manipulators. We describe: (1) the collision-free trajectory optimization formulation and GCDF-based constraint handling to ensure gradient and value validity; (2) the numerical algorithm for solving the nonlinear program and our efficient implementation that bridges the PyTorch model with a C++ solver by exploiting parallelization and constraint sparsity; and (3) a perception-integrated navigation framework that enables continuous, safe, and fast online replanning for long-range navigation in unknown environments.
\subsection{Collision-Free Trajectory Optimization for Mobile Manipulator}
We now formalize the trajectory optimization problem for a mobile manipulator system composed of a mobile base and an articulated robot arm.

\textbf{Problem Formulation.} Denote the optimization horizon as $N$, and $\bm{q}_i \in \Real{n}$ the configuration at the $i$th step for $\forall i = 1, 2, \dots, N$. By separating the configuration variables into base translational and rotational DoFs, and incorporating velocity control variables, we define the full state variable at time step $i$ as:
\begin{equation}
        \bm x_i = \underbrace{[\boldsymbol{q}_{i}^{t}, \boldsymbol{q}_{i}^{r},}_{\textup{configuration}}\underbrace{\boldsymbol{v}_{i}^{t}, \boldsymbol{v}_{i}^{r}]}_{\textup{velocity}}\in \Real{2n}.
\end{equation}
The translational components (with superscript $t$) describe the mobile base's planar motion in the $x$-$y$ plane, while the rotational components (with superscript $r$) correspond to the base orientation and manipulator joint angles. Let us further denote the first elements of the rotational joint angles and velocities, $q^r_{i,0}$ and $v_{i,0}^r$, as the base rotation angle and angular velocity, respectively. By collecting the state variables across all time steps into $\calX = \{\bm{x}_1, \bm{x}_2,\dots, \bm{x}_N\}$, the optimization problem takes the following form:
\begin{subequations}\label{eq:optimization-in-general}
\begin{align}
    \displaystyle  \operatorname*{min}_{\mathcal{X}\in\Real{2nN}}\quad & J_c + J_t \label{eq:optimization-in-general:cost}\\ 
    \operatorname*{subject\ to}\quad \,\, & \boldsymbol{q}_{i+1,x}^t=\boldsymbol{q}_{i,x}^t + (v_{i,x}^t  \cos q_{i,0}^r - v_{i,y}^t \sin q_{i,0}^r)dt,\nonumber\\
    & \boldsymbol{q}_{i+1,y}^t=\boldsymbol{q}_{i,y}^t + (v_{i,x}^t  \sin q_{i,0}^r + v_{i,y}^t \cos q_{i,0}^r)dt,\nonumber\\ 
    & \boldsymbol{q}^r_{i+1}=\boldsymbol{q}^r_{i} + \bm v_i^rdt,\nonumber\\
     &\quad\quad \forall i = 1,2,\ldots,N-1\label{eq:optimization-in-general:kinematics}\\
    & \bm h(\boldsymbol{x}_{i}) = \zero,\nonumber\\ &\bm g(\boldsymbol{x}_{i}) \geq \zero,\nonumber\\
    &\quad\quad \forall i = 1,2,\ldots,N\label{eq:optimization-in-general:in-eq-constraints}\\
    & f^g_c(\boldsymbol{p}_{j},\bm q_{i}) - \delta \geq 0,\nonumber\\
    &\quad\quad \forall i = 1,2,\ldots,N, \forall j = 1,2,\ldots, M \label{eq:optimization-in-general:cdf-constraint}\\
    &\boldsymbol{x}_0 = \boldsymbol{x}_{s}\label{eq:optimization-in-general:initial-constraint}.
\end{align}
\end{subequations}

The objective function (\ref{eq:optimization-in-general:cost}) consists of two terms: $J_c$ represents the control cost and $J_t$ denotes the tracking error for desired states, both formulated as standard quadratic costs in our case. The velocity-controlled system is governed by the kinematic constraints~(\ref{eq:optimization-in-general:kinematics}): The first two equations describe the translational motion of the mobile base, while the last one updates all rotational joints, including both the base orientation and manipulator joint angles. The general equality and inequality constraints are organized as $h$ and $g$ in (\ref{eq:optimization-in-general:in-eq-constraints}), where the comparisons are element-wise. Importantly, safety constraints are enforced through (\ref{eq:optimization-in-general:cdf-constraint}), which generates $N \times M$ constraints ensuring that the GCDF value between the robot and each of the $M$ obstacle points remains above a safety threshold $\delta$ throughout the $N$-step optimization horizon. The initial condition is specified by (\ref{eq:optimization-in-general:initial-constraint}).  
\begin{remark}
In practice, the safety threshold $\delta$ must be chosen to balance neural approximation error and avoidance conservativeness. If $\delta$ is too small, it may not compensate for the approximation error of the neural implicit function, potentially violating collision avoidance. If $\delta$ is too large, the resulting behavior becomes overly conservative and may render the problem infeasible in narrow environments.
\end{remark}

\textbf{Algorithm. }To solve the nonlinear programming problem~\eqref{eq:optimization-in-general}, we adopt the sequential convex optimization algorithm, which iteratively approximates a local model (first-order constraint information and second-order objective information) around the current iterate $\mathcal{X}_k$ and solves a convex subproblem at each step. We chose this algorithm for the following reasons. First, as discussed in Section~\ref{sec:method:cdf-mobile-manipulator}, the trained neural GCDF directly exploits distance information in configuration space and explicitly maintains accurate gradient information that ensures the quality of local linearization. This accurate local model enables larger effective step sizes within the trust region or longer line search steps at each iteration, leading to faster convergence. Second, considering the large number of GCDF constraints between the robot and each obstacle point in~\eqref{eq:optimization-in-general}, sequential convex optimization offers computational efficiency through well-established large-scale sparsity-aware QP subproblem solvers. This makes our approach particularly suitable for fast trajectory planning in obstacle-dense environments where the number of constraints can scale significantly.

Specifically, at the $k$-th iteration, we approximate a local model around the current iterate $\mathcal{X}_k$ and solve the following $\ell_1$ penalty with $\ell_{\infty}$ trust region convex subproblem, following the formulation in \crisp~\cite{Li2025RSS-crisp}:
\begin{subequations} \label{eq:subproblem-smooth}
    \begin{align}
        \min_{p_k, v, w, t} & \displaystyle \,\,\,J_k + \nabla J_k\tran p_k  + \frac{1}{2} p_k\tran \nabla_{xx}^2 J_k p_k \nonumber \\
        & + \mu \sum_{i \in \mathcal{E}} (v_i + w_i) + \mu \sum_{i \in \mathcal{I}} t_i  \\
        \text{subject to} &\,\, \nabla c_i(x_k)\tran p_k + c_i(x_k) = v_i - w_i, \ i \in \mathcal{E} \label{eq:subproblem-smooth-eq}\\
        & \,\, \nabla c_i(x_k)\tran p_k + c_i(x_k) \ge -t_i, \ i \in \mathcal{I} \label{eq:subproblem-smooth-ineq}\\
        & \,\, v, w, t \ge 0 \\
        & \,\, \|p_k\|_\infty \le \Delta_k. \label{eq:subproblem-trust-region}
    \end{align}
\end{subequations}
Here, we denote the sum of $J_c$ and $J_t$ in~\eqref{eq:optimization-in-general} as $J$, and stack all the equality and inequality constraints in $c$ with the corresponding index sets $\mathcal{E}$ and $\mathcal{I}$.
$J_k$, $\nabla J_k$, and $\nabla_{xx}^2 J_k$ represent the objective function's value, gradient, and Hessian at the current iterate $\mathcal{X}_k$, respectively. The nonnegative slack variables $v,w,t$ are introduced to penalize constraint violations. The constraints $c_i$ are linearized around $\mathcal{X}_k$, and $p_k$ is the current trial step subject to the trust region constraint in~\eqref{eq:subproblem-trust-region} to find an optimal step within a local range. To obtain accurate GCDF queries aligned with our training distribution, we transform obstacle points into the robot's local base frame at each time step. Specifically, we treat the robot base pose as the origin and bias all obstacle points $\bm{p}$ into this coordinate frame before querying the neural GCDF. This ensures queries remain within the coverage of our training data, providing reliable distance estimates and gradients for optimization.

\begin{remark}
    Advanced solvers for such large-scale optimization problems have increasingly exploited problem-specific sparsity structures to improve computational efficiency. For instance, the chain-like sparsity in the system dynamics~\cite{kang2024fast}. These solvers typically rely on one of two approaches: either using automatic differentiation tools (e.g., CppAD) to compute sparse Jacobians of explicitly defined functions~\cite{Li2025RSS-crisp,OCS2,carpentier2019pinocchio}, or requiring manual specification of sparsity patterns for constraint mapping~\cite{drake}. However, our GCDF-based collision avoidance constraints are defined implicitly through neural networks, making direct application of these approaches challenging. Moreover, the sparsity pattern of GCDF constraints varies with different obstacle configurations and robot query point distributions, requiring a flexible representation that can be efficiently reconstructed online.
\end{remark}
In this work, we choose \crisp~\cite{Li2025RSS-crisp} as our backbone solver for its computational efficiency with CppAD-based automatic differentiation, capability to generate informed trajectories from naive initial guesses, and trust-region framework that naturally leverages GCDF's accurate gradient information.
To integrate our implicit neural collision constraints with \crisp, we extend the solver with the following key modifications. First, we enable GPU-parallel querying of implicit function values and gradients during the solving process. The gradients from parallel GCDF queries form a dense matrix in $\mathbb{R}^{NM \times n}$. Second, we design a compact representation for the structure of GCDF constraints that enables online reconstruction of constraint mappings and automatic derivation of their sparsity patterns. This representation captures the relationship between robot query points, obstacle points, and configuration variables, allowing efficient adaptation to varying obstacle configurations. Third, we implement efficient memory-level sparse mapping from the dense GCDF gradient matrix into the overall problem's sparse Jacobian matrix. This process, combined with GPU-parallel querying, significantly improves computational efficiency.

\textbf{Implementation. }
To enable cross-platform GPU querying of the trained neural network, we first convert the PyTorch GCDF model into a batch-queryable CasADi~\cite{andersson19mpc-casadi} function using the L4CasADi library~\cite{salzmann2024l4casadi}, which automatically computes gradients through CasADi's symbolic differentiation. We then generate C++ code from this CasADi function and compile it into a dynamic library. Leveraging CasADi's cross-platform compatibility, our solver dynamically loads this library during the online phase, providing efficient inference functions for both GCDF values and gradients that can be seamlessly integrated into the optimization loop. The implementation of this pipeline involves numerous GPU-related and environment-specific engineering challenges. Detailed instructions and source code are provided on our project website.
\begin{remark}
In obstacle-dense environments, the GCDF collision avoidance constraints constitute the largest component of the optimization problem~(\ref{eq:optimization-in-general}). Beyond enabling the solver to efficiently obtain constraint values and Jacobians through GPU batch queries, it is crucial to exploit the sparsity structure of these constraints.    
\end{remark}
 For the GCDF constraints in~(\ref{eq:optimization-in-general:cdf-constraint}), where each obstacle point is considered for all time steps, this would result in $N \times M$ constraints, which is a prohibitively large number. For mobile manipulation tasks, since the workspace is relatively larger due to base movements, it is clearly unnecessary to consider all possible collision pairs. We apply a straightforward approach to partition obstacle points based on the reference base positions, only considering obstacle points within a certain range of each reference base position. To this end, we design a straightforward data structure consisting of two components: one storing the spatial coordinates of all obstacle points, and another storing the indices of obstacle points to be considered at each time step. This structure enables efficient retrieval of the number of constraints and the sparsity pattern of the obstacle constraints. Moreover, both components can be modified online to accommodate changes in obstacle layouts, allowing for dynamic updates to the sparsity pattern and constraint count during online replanning. Let us now mathematically formulate this partitioning scheme.
\begin{figure}[t]
\setlength{\belowcaptionskip}{-10pt} 
\centering
\includegraphics[width=0.5\textwidth, trim={0cm 0cm 0cm 0cm}, clip]{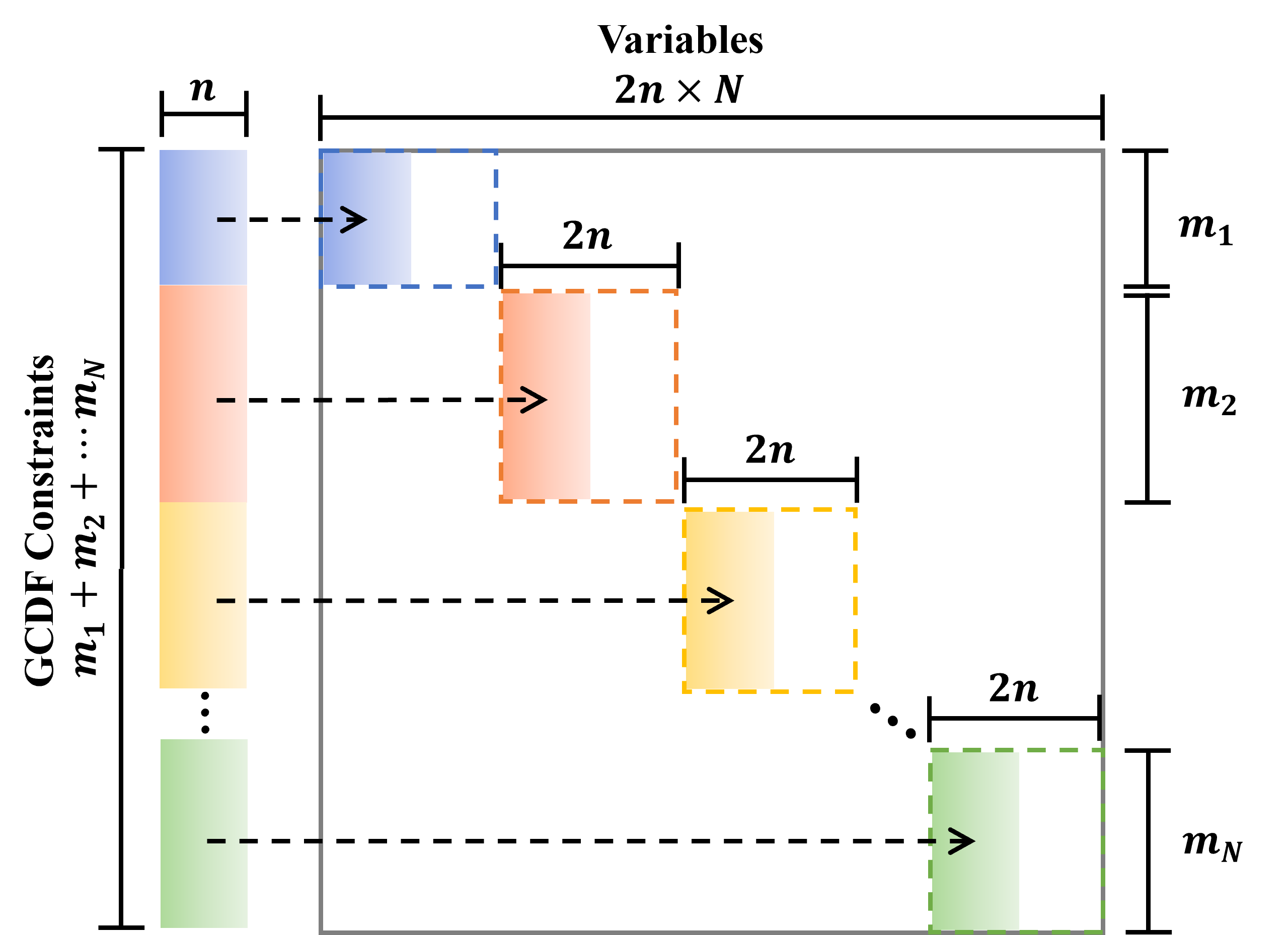}
\caption{Visualization of the Jacobian sparsity pattern induced by GCDF constraints, together with the dense-to-sparse mapping that assembles the full GCDF gradient vector queried from the implicit neural model into the sparse Jacobian.}
\label{fig:sparsity-pattern}
\end{figure}

Suppose the obstacle set $\mathbb{I}_{M}$ contains $M$ points. We divide it into $N$ partitions, denoted by $\mathbb{I}_{M,i}$ for $i \in 1,2,\ldots N$, representing potential collision points at each step. This gives us:
\begin{equation}
    \mathbb{I}_M \subseteq \mathbb{I}_{M,1} \cup \mathbb{I}_{M,2} \cup \ldots \cup \mathbb{I}_{M,N}, 
\end{equation}
where these sets typically have overlaps, with the $i$th set containing $m_i$ points. Let us analyze all GCDF constraints and its Jacobian matrix in vectorized form. The GCDF constraints can be organized in $\bm c_{\text{gcdf}}$ as :
\begin{equation}
\label{eq:cdf_constraints_vector}
    \bm c_{\text{gcdf}} = 
\begin{bmatrix}
    f^g_c(\bm p_{1,1}, \bm q_1) \\
    \vdots \\
    f^g_c(\bm p_{1,m_1}, \bm q_1) \\
    \vdots \\
    f^g_c(\bm p_{i,1}, \bm q_i) \\
    \vdots \\
    f^g_c(\bm p_{i,m_i}, \bm q_i) \\
    \vdots \\
    f^g_c(\bm p_{N,1}, \bm q_N) \\
    \vdots \\
    f^g_c(\bm p_{N,m_N}, \bm q_N)
\end{bmatrix} = 
\begin{bmatrix}
\bm c_{\bm q_1} \\\vdots \\\bm c_{\bm q_i} \\ \vdots\\\bm c_{\bm q_N} \\    
\end{bmatrix} \in \Real{m_1 + m_2 + \ldots + m_N}.
\end{equation}
For simplicity, we omit the constant term $\delta$. The constraints are indexed by time step, where obstacle points $\bm p_{i,1}$ to $\bm p_{i,m_i}$ belonging to $I_{M,i}$ are considered for $\bm q_i$, with related constraints stacked in $\bm c_{\bm q_i}$ as shown in the second equality. While $\bm c_{\text{gcdf}}$ is a dense vector that can be efficiently queried using batch operations, its Jacobian matrix exhibits high sparsity due to the partition. This sparsity pattern must be carefully exploited in the optimization process to ensure efficient solving of the problem (\ref{eq:optimization-in-general}).

Specifically, the Jacobian matrix $\nabla^\top_{\bm q} \bm c_{\text{gcdf}}$ is sparse with dimension $(m_1 + m_2 + \ldots + m_N) \times 2Nn$, and its sparsity pattern is visualized in Fig.~\ref{fig:sparsity-pattern}. While the intermediate gradient matrix obtained directly from the implicit neural CDF, denoted as $\nabla^\top_{\bm q} \Tilde{\bm c}$, has dimension $(m_1 + m_2 + \ldots + m_N) \times n$. We need to project values from $\nabla^\top_{\bm q} \Tilde{\bm c}$ into their proper positions in the sparse matrix $\nabla^\top_{\bm q} \bm c_{\text{gcdf}}$. This process can be formalized as:
\begin{equation}
\nabla^\top_{\bm q} \Tilde{\bm c} = 
\begin{bmatrix}
\nabla^\top_{\bm q_1}\bm c_{\bm q_1} \\\vdots \\\nabla^\top_{\bm q_i} \bm c_{\bm q_i} \\ \vdots\\\nabla^\top_{\bm q_N} \bm c_{\bm q_N} \\    
\end{bmatrix} \in \Real{(m_1 + m_2 +\ldots + m_N) \times n},
\end{equation}
where
\begin{equation}
\nabla^\top_{\bm q_i} \bm c_{\bm q_i} \in \Real{m_i \times n}. \quad\forall i = 1,2,\ldots N
\end{equation}
We now show that this mapping can be treated as dense-sparse matrix multiplication, we construct a sequence of $N$ sparse projection matrices $\bm P = \{\bm P_1, \bm P_2,\ldots \bm P_N\}$. The projection can be achieved through block operations:
\begin{equation}
\label{eq:cdf_jacobian_construction}
\nabla^\top_{\bm q} \bm c = \begin{bmatrix}
\nabla^\top_{\bm q_1} \bm c_{\bm q_1} \bm P_1 \\\vdots \\\nabla^\top_{\bm q_i} \bm c_{\bm q_i} \bm P_i \\ \vdots\\\nabla^\top_{\bm q_N} \bm c_{\bm q_N} \bm P_N \\    
\end{bmatrix} \in \Real{(m_1 + m_2 +\ldots + m_N) \times 2Nn},
\end{equation}
where the projection matrices $\bm P_i \in \Real{n\times2Nn}$ are defined as:
\begin{equation}
(\bm P_i)_{r,c} = 
\begin{cases} 
1, & \text{if } r = c - 2Nn(i-1)\\
0, & \text{otherwise}
\end{cases}
\end{equation}
As shown in Fig.~\ref{fig:sparsity-pattern}, in essence, $\bm P_i$ is constructed by placing an identity matrix $I_n$ at row 1, column $2Nn(i-1) + 1$, effectively positioning the $n$ elements of each row of $\nabla^\top_{\bm q_i} \bm c_{\bm q_i}$ to their proper positions.
\begin{remark}
     It is worth noting that the focus is on utilizing the sparsity pattern to develop high-performance tailored numerical methods, rather than specific implementation details. For instance, while MATLAB would benefit from the dense-sparse matrix multiplications shown in (\ref{eq:cdf_jacobian_construction}), our C++ implementations use direct memory operations to position values correctly.
\end{remark}

\section{Experiments}
\label{sec:exp}
In this section, we design experiments to comprehensively evaluate our framework from three complementary perspectives. First, we validate the theoretical properties of GCDF and the effectiveness of our training pipeline by assessing the learned neural GCDF. In particular, we examine whether the properties in Theorem~\ref{thm:gcdf_properties} hold in practice and whether the neural GCDF provides sufficiently accurate values and gradients to support the single-step projection, which maps a configuration onto the queried obstacle point's zero-level set using first-order information. Second, we benchmark our GCDF-constrained trajectory optimization algorithm against the state-of-the-art (SOTA) pipelines in multiple randomized maps to demonstrate efficiency, robustness, and solution quality in obstacle-dense environments. Third, we deploy the full system on a real mobile manipulator to verify practicality in real-world scenes.
\begin{figure}[t]
\setlength{\belowcaptionskip}{-10pt} 
\centering 
\includegraphics[width=0.5\textwidth, trim={0cm 0cm 0cm 0cm}, clip]{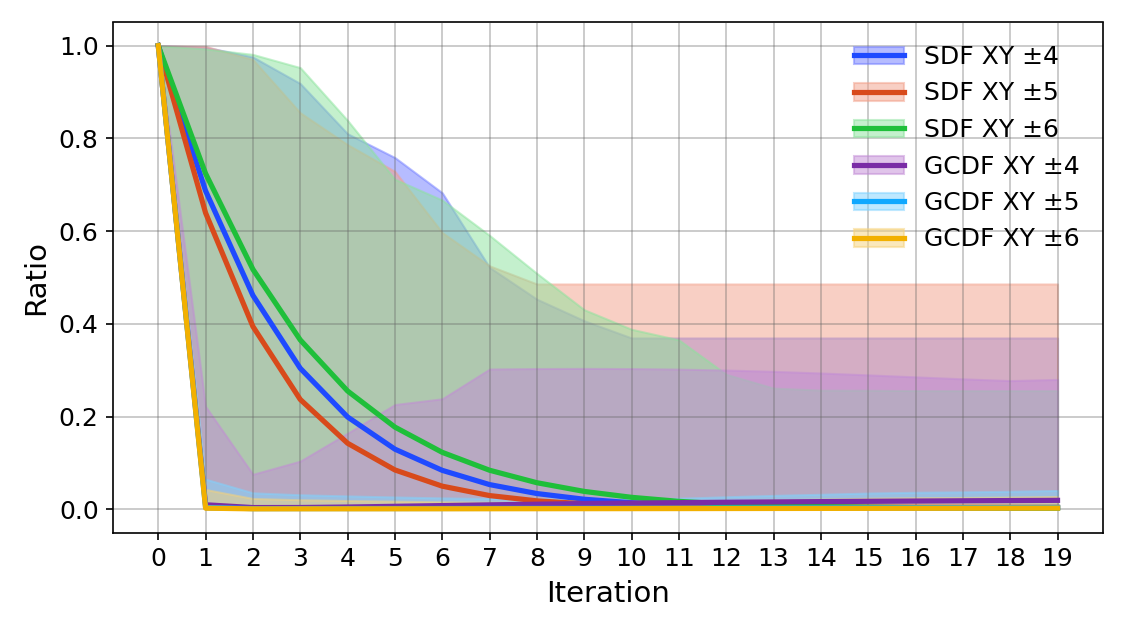}
\caption{Residual reduction comparison between our GCDF-based analytical projection and an iterative workspace SDF baseline. Targets are sampled within three $x$--$y$ ranges ($\pm4\,\mathrm{m}$, $\pm5\,\mathrm{m}$, $\pm6\,\mathrm{m}$), and results aggregate 128 targets with 128 random arm initializations per target. Shaded regions indicate variability across trials.}

\label{fig:projection_error}
\end{figure}
\begin{figure}[!t]
\setlength{\belowcaptionskip}{-10pt} 
\centering
\includegraphics[width=0.5\textwidth, trim={0cm 1cm 15cm 0cm}, clip]{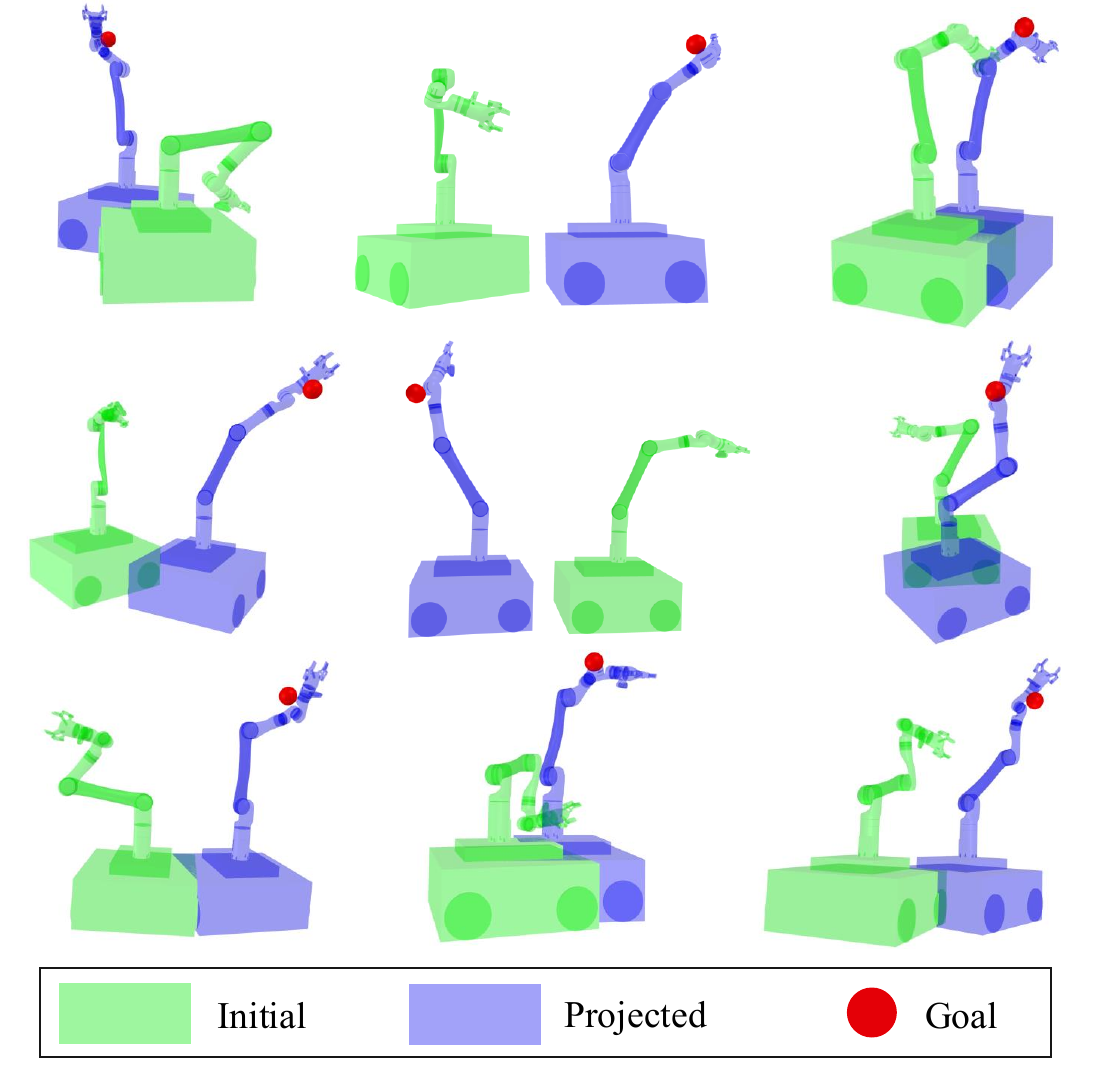}
\caption{Visualization of the analytical single-step projection using the trained neural GCDF. The red dot denotes the queried target point. The initial configuration and the projected configuration after one step are overlaid for comparison. The result highlights the accuracy of the learned GCDF (values/gradients) and shows that the proposed projection moves directly in configuration space toward the target point's nearest zero-level set.}
\label{fig:projection}
\end{figure}
\begin{figure*}[!h]
\setlength{\belowcaptionskip}{-10pt} 
\centering
\includegraphics[width=\textwidth, trim={0cm 0cm 0cm 0cm}, clip]{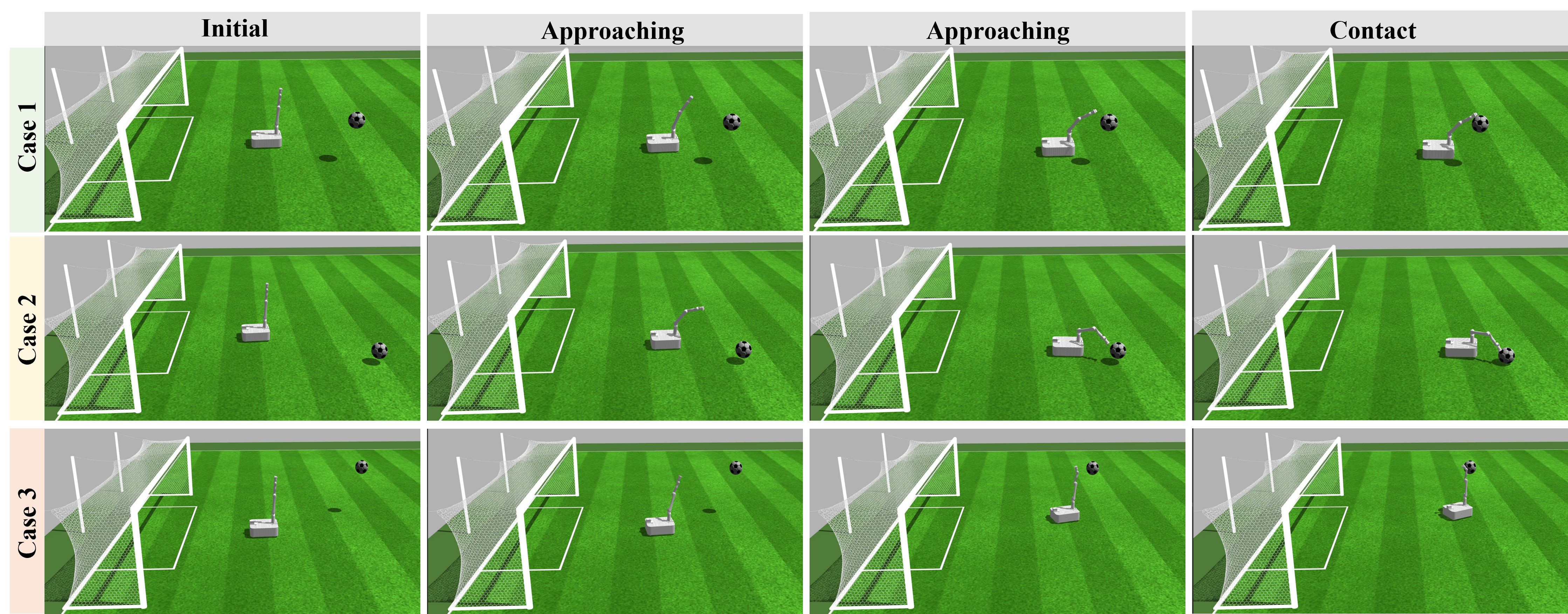}
\caption{Visualization of the goalkeeper example with online GCDF-based control. Each row shows a different shot with a distinct initial ball state; frames progress left-to-right in time. At each time step, the controller evaluates the value and gradient of the GCDF online and applies the corresponding projection to generate real-time, collision-aware motions that intercept the incoming ball.}
\label{fig:goal_keeper}
\end{figure*}
\subsection{Simulation Experiments}
\subsubsection{Training Results}
We begin by reporting the training settings and outcomes for GCDF. We set the loss weights to $\lambda_1 = 5.0$, $\lambda_2 = 0.1$, $\lambda_3 = 0.01$, and $\lambda_4 = 0.01$, and optimize the network by minimizing $\mathcal{L}_{\text{total}}$ in (\ref{eq:total_lo}) to achieve accurate GCDF approximation while enforcing the desired regularization.
For each training iteration, we sample $b_1=20$ environment points from the grid and, for each point, draw $b_2=100$ configurations, resulting in a batch size of $20\times100$. We train for 14{,}900 epochs using Adam with an initial learning rate of 0.005, decayed by a factor of 0.5. Training takes approximately 2 hours on four NVIDIA RTX 3090 GPUs. The final loss terms are $\mathcal{L}_{\text{dist}}=0.274$, $\mathcal{L}_{\text{grad}}=0.100$, $\mathcal{L}_{\text{eikonal}}=0.228$, and $\mathcal{L}_{\text{tension}}=18.571$. Unless otherwise specified, we use this model throughout the following experiments. The training code and pretrained weights are available on the project website.
\subsubsection{Properties of Neural GCDF}
Next, we validate the Euclidean projection property and evaluate the accuracy of the value and gradient information provided by the learned implicit GCDF. Preserving this property with accurate first-order information is a core theoretical pillar of our framework: it underpins the use of GCDF values and gradients to enforce collision avoidance and perform trajectory optimization directly in configuration space, thereby avoiding the intricate nonlinear mapping and kinematic coupling between workspace collision constraints and configuration space variables.

Concretely, we quantitatively evaluate this property by comparing the residual reduction achieved by our GCDF-based analytical projection against iterative algorithm using workspace SDF directly. Specifically, we randomly sample 128 target environment points within three $x$–$y$ ranges ($\pm 4\,\textup{m}$, $\pm 5\,\textup{m}$, and $\pm 6\,\textup{m}$) around the mobile base to cover varying proximity and difficulty. For each target point $\bm p$, we initialize the robot from 128 different arm configurations $\bm q$ and apply the projection in Theorem~\ref{thm:gcdf_properties} to obtain a projected configuration $\bm q^{+}$. We quantify convergence by measuring the residual ratio between the initial and projected states, computed using SDF values evaluated at the corresponding configurations. Since the projection is closed-form, its runtime is negligible. As a baseline, we implement standard gradient descent with Armijo line search to solve the unconstrained problem $\min_{\bm q} f_s^2(\bm q,\bm p)$, and record the residual ratio after each iteration. The results are shown in Fig.~\ref{fig:projection_error}. Our GCDF projection brings the robot to (or extremely close to) the target point’s zero-level set in a single step across all ranges, whereas the SDF-based method often requires multiple iterations. This gap is as expected: \textit{the SDF gradient must be back-propagated through nonlinear forward kinematics, so its induced updates in configuration space are strongly coupled and effective only locally, leading to slower contraction toward the contact manifold.} Some results of our single-step projection are visualized 
in Fig.~\ref{fig:projection}.

Additionally, we design a goalkeeper scenario as an intuitive demonstration of the practical utility of GCDF projection in real-time closed-loop control. In this example, a ball is launched toward the goal mouth from different initial states, and the mobile manipulator reacts online by querying GCDF values and gradients and using the resulting projection as a time-varying control target. The target is recomputed at every timestep, yielding smooth and continuous motions that drive the robot toward interception. As visualized in Fig.~\ref{fig:goal_keeper}, the robot approaches the incoming ball and makes contact to block the shot, preventing a goal. 
\subsubsection{Benchmark Comparisons}
In this part, we conduct a comprehensive benchmark study to demonstrate the effectiveness and superiority of the proposed trajectory optimization algorithm in complex environments.
\begin{figure}[!t]
\setlength{\belowcaptionskip}{-10pt} 
\centering
\includegraphics[width=0.5\textwidth, trim={0cm 0cm 0cm 0cm}, clip]{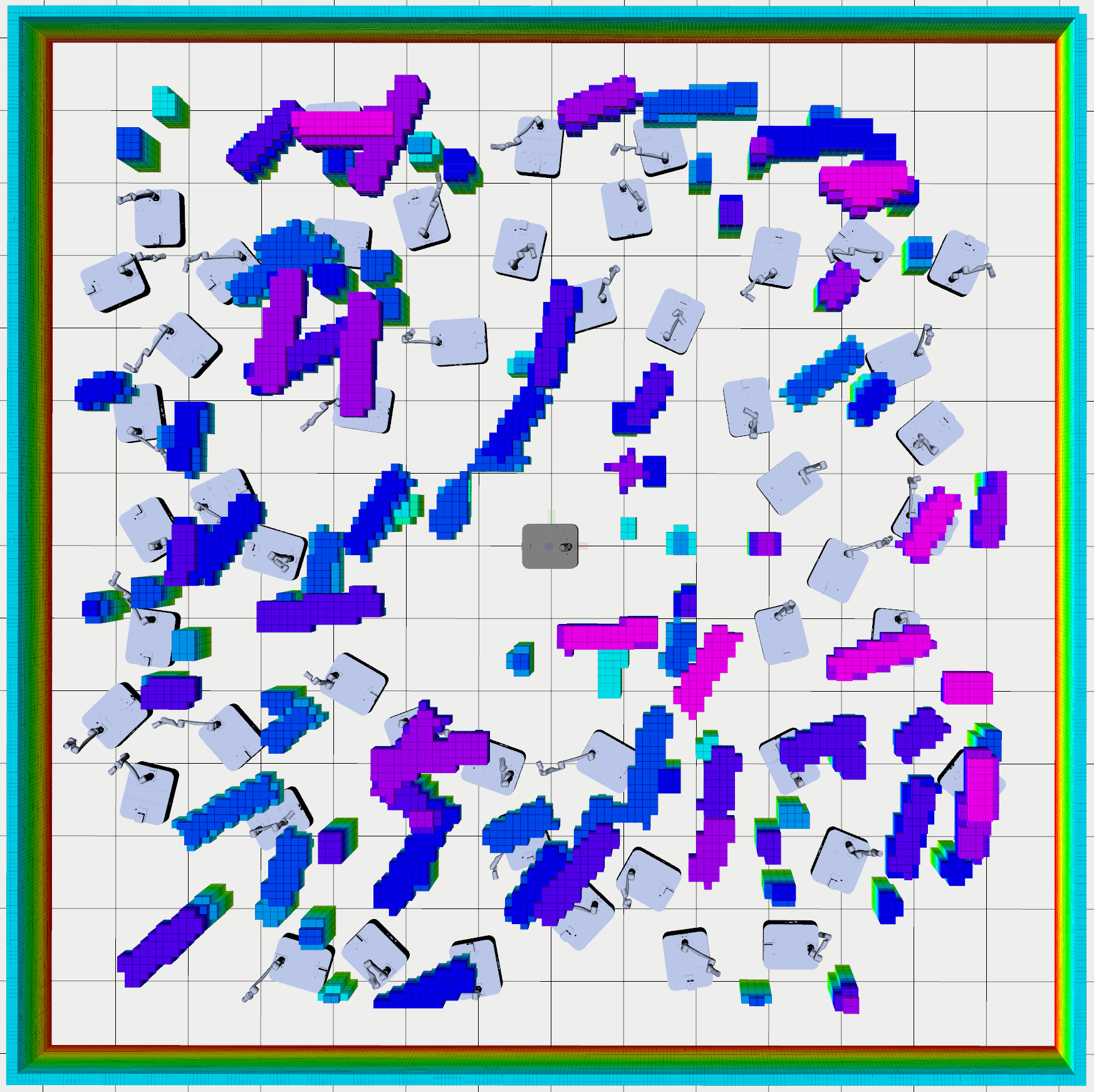}
\caption{Environment setting for benchmark evaluation. We generate 50 random feasible start–goal configurations in each map. The robot arm is a 6-DoF Kinova Gen3.}
\label{fig:benchmark_env_setting}
\end{figure}

\textbf{Environment Settings. }
As illustrated in Fig.~\ref{fig:benchmark_env_setting}, we construct a $14\,\textup{m}\times14\,\textup{m}$ map and randomly populate it with a prescribed number of rectangular obstacles placed at varying heights. We remark that our setup is substantially more challenging than commonly used discrete pillar-forest environments: obstacles have diverse sizes and configurations, are allowed to overlap, and their vertical stacking naturally induces highly non-convex geometry across different height levels, which poses significant difficulty for coordinated base–arm motion. We consider three difficulty tiers by varying the obstacle count to $80$, $100$, and $120$, respectively. The robot starts at the map center with the arm initialized to the zero configuration (upright). We randomly sample 50 collision-free and kinematically feasible goal whole-body poses whose base locations are at least $3\,\textup{m}$ away from the center. For each difficulty tier, we generate 5 random maps, resulting in $50\times5=250$ planning trials per tier.

\begin{figure}[t]
\setlength{\belowcaptionskip}{-15pt}
\centering
\includegraphics[width=\textwidth, trim={0cm 0.5cm 6cm 0cm}, clip]{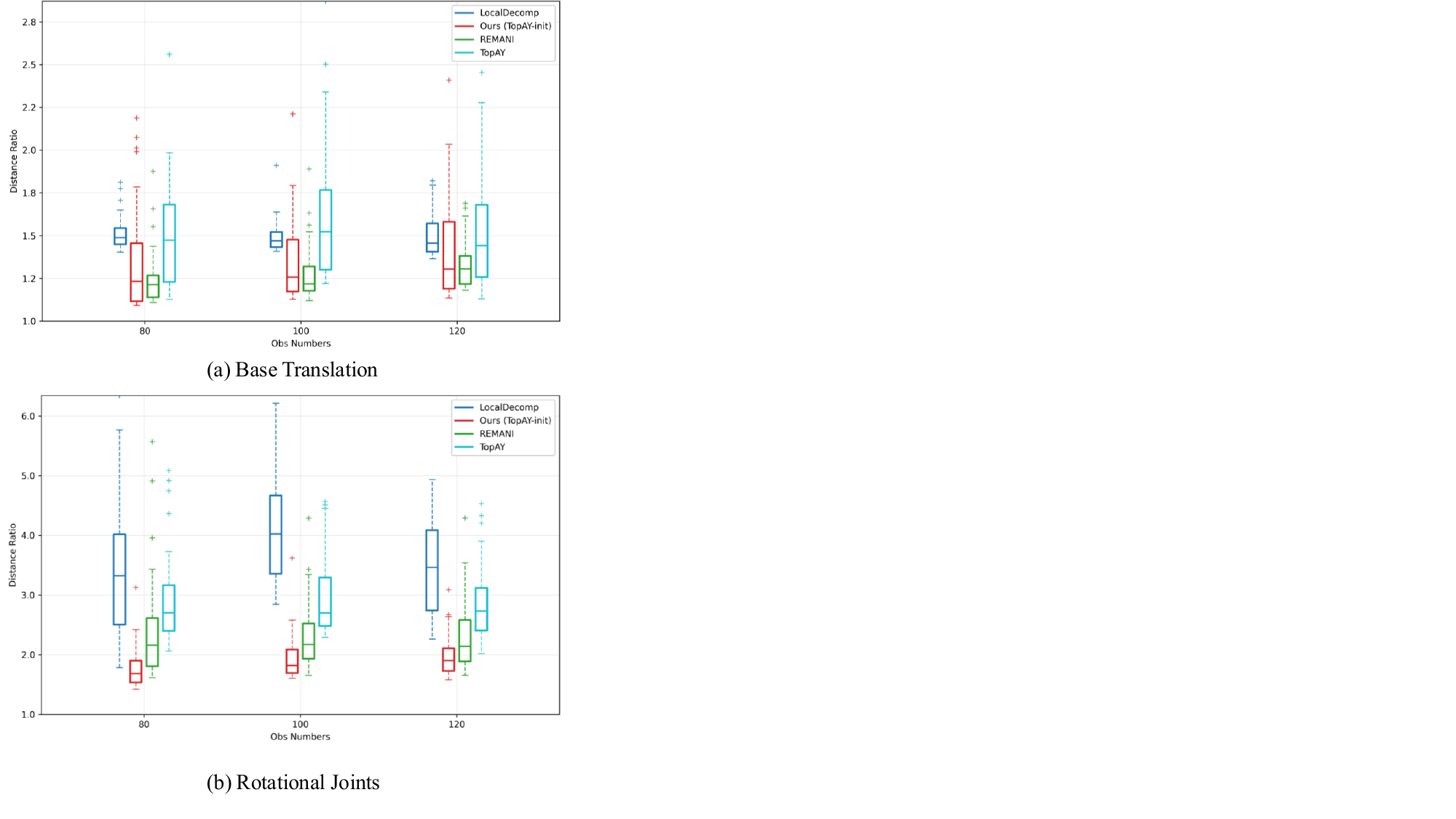}
\caption{Boxplots of configuration-space path-length ratios in randomized clutter benchmarks. Ratios are computed over successful trials as each baseline's trajectory length normalized by ours (lower is better) under three obstacle densities (80/100/120): (a) base translation component and (b) rotational joint component.}
\label{fig:c_space_ratio}
\end{figure}

\begin{figure*}[h!tbp]
\setlength{\belowcaptionskip}{20pt}
\centering
\includegraphics[width=\textwidth, trim={0cm 0cm 0cm 0cm}, clip]{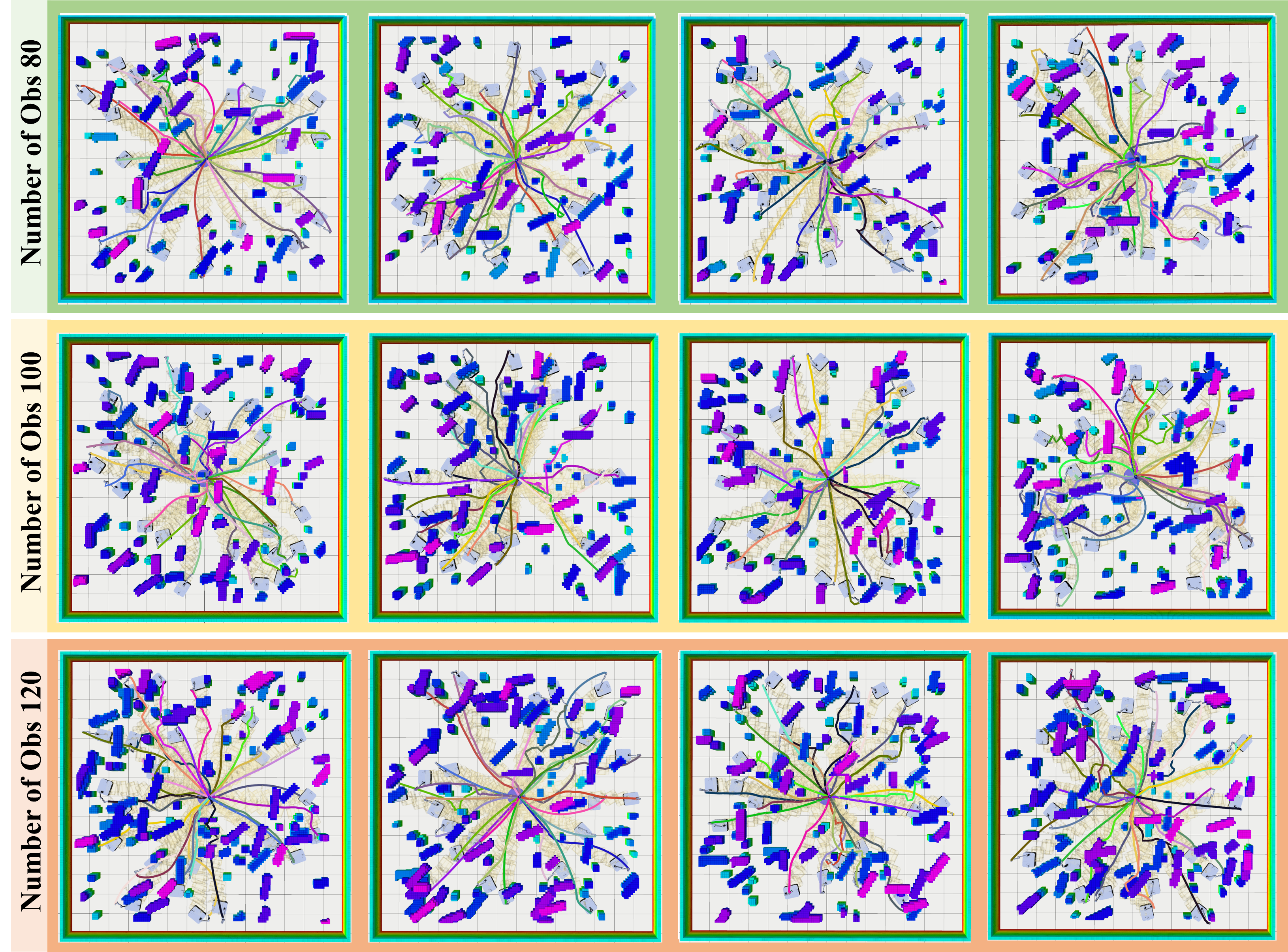}
\caption{Visualization of representative trajectories generated by \textbf{Ours} in randomized maps. For each map, we overlay 24 planned motions (different goals) and plot the corresponding end-effector traces. Rows correspond to obstacle densities (80/100/120). The end-effector paths reveal agile whole-body behaviors enabled by configuration-space optimization, including coordinated base--arm reconfiguration and non-conservative twisting maneuvers that pass through narrow, highly non-convex regions. Additional visualizations from multiple viewpoints are available on the project website.}
\label{fig:benchmark_simulation_all}
\end{figure*}
\begin{table*}[h!tbp]
\centering
\begin{threeparttable}
\small                       
\setlength{\tabcolsep}{5.6pt}
\renewcommand{\arraystretch}{1.28}

\sisetup{
  detect-weight=true,
  detect-family=true,
  table-number-alignment=center
}

\caption{\textbf{Benchmark results under different obstacle densities.}
We report success rate (SR), average runtime, and the average traversed trajectory length ratio
w.r.t.\ our method (Ours $=\textbf{1.00}$), separated into translational and rotational components.}
\label{tab:benchmark_results}

\begin{tabular*}{\textwidth}{@{\extracolsep{\fill}} c
  c S[table-format=1.4] S[table-format=1.2] S[table-format=1.2]
  c S[table-format=1.4] S[table-format=1.2] S[table-format=1.2]
  c S[table-format=1.4] S[table-format=1.2] S[table-format=1.2] @{}}

\specialrule{1.0pt}{0pt}{0.4pt}
\multirow[c]{3}{*}{\textbf{Methods}}
& \multicolumn{4}{c}{\textbf{Number of Obstacles: 80}}
& \multicolumn{4}{c}{\textbf{Number of Obstacles: 100}}
& \multicolumn{4}{c}{\textbf{Number of Obstacles: 120}} \\
\cmidrule(lr){2-5}\cmidrule(lr){6-9}\cmidrule(lr){10-13}

& \multicolumn{1}{c}{\makecell{\textbf{SR}\\(\%)}}
& \multicolumn{1}{c}{\makecell{\textbf{Time}\\(s)}}
& \multicolumn{2}{c}{\makecell{\textbf{Ratio} ($\downarrow$)\\\footnotesize Trans\hspace{10pt}Rot}}
& \multicolumn{1}{c}{\makecell{\textbf{SR}\\(\%)}}
& \multicolumn{1}{c}{\makecell{\textbf{Time}\\(s)}}
& \multicolumn{2}{c}{\makecell{\textbf{Ratio} ($\downarrow$)\\\footnotesize Trans\hspace{10pt}Rot}}
& \multicolumn{1}{c}{\makecell{\textbf{SR}\\(\%)}}
& \multicolumn{1}{c}{\makecell{\textbf{Time}\\(s)}}
& \multicolumn{2}{c}{\makecell{\textbf{Ratio} ($\downarrow$)\\\footnotesize Trans\hspace{10pt}Rot}} \\
\midrule


\scenario{LocalDecomp}
& 86.40 & {\textbf{---}} & 1.49 & 3.34
& 59.60 & {\textbf{---}} & 1.47 & 4.10
& 56.40& {\textbf{---}} & 1.45 & 3.45 \\

\scenario{REMANI}
& 86.80 & 0.2858 & \textcolor{orange}{1.23} & \textcolor{orange}{2.15}
& 61.50 & 0.3077 & \textcolor{orange}{1.24} & \textcolor{orange}{2.21}
& 49.60 & 0.3292 & 1.33 & \textcolor{orange}{2.22} \\

\scenario{TopAY}
& 88.40 & \textcolor{orange}{0.1393} & 1.49 & 2.71
& 67.20 & \textcolor{orange}{0.1776} & 1.51 & 2.72
& 51.60 & \textcolor{orange}{0.2727} & 1.47 & 2.78 \\

\textbf{Ours}
& \textcolor{orange}{92.40} & 0.1732 & \textbf{1.00} &  \textbf{1.00}
&  \textcolor{orange}{88.40}& 0.2446 & \textbf{1.00} & \textbf{1.00}
& \textcolor{orange}{86.80} & 0.2871 & \textbf{1.00} & \textbf{1.00} \\

\textbf{Ours (TopAY-init)}
& \best{98.00} & \best{0.1341} & \textcolor{orange}{1.25} & \best{1.81}
& \best{92.50} & \best{0.1641} & \best{1.28} & \best{1.92}
& \best{91.60} & \best{0.1957} & \best{1.32} & \best{1.98} \\

\specialrule{1.0pt}{0.4pt}{0pt}
\end{tabular*}
\begin{tablenotes}[flushleft]
\footnotesize
\item[*] \textbf{---} denotes that the runtime is \emph{not directly comparable} because the method is a \emph{single-step local controller}. 
 \item[*]\best{Dark Green} highlights the best results while \textcolor{orange}{Orange} indicates the second-best. We report the runtime for \emph{all} trials, while the Ratio is computed only over \emph{successful} trajectories.
\item[*] \textbf{Ours (TopAY-init)} initializes our solver with the discrete front-end paths produced by TopAY; if the front-end search fails, we fall back to the original baseline initialization, i.e., linear interpolation for the mobile base and an all-zero arm configuration.
\end{tablenotes}
\end{threeparttable}
\end{table*}
\begin{figure*}[!t]
\setlength{\belowcaptionskip}{-10pt}
\centering
\includegraphics[width=\textwidth, trim={0cm 0cm 0cm 0cm}, clip]{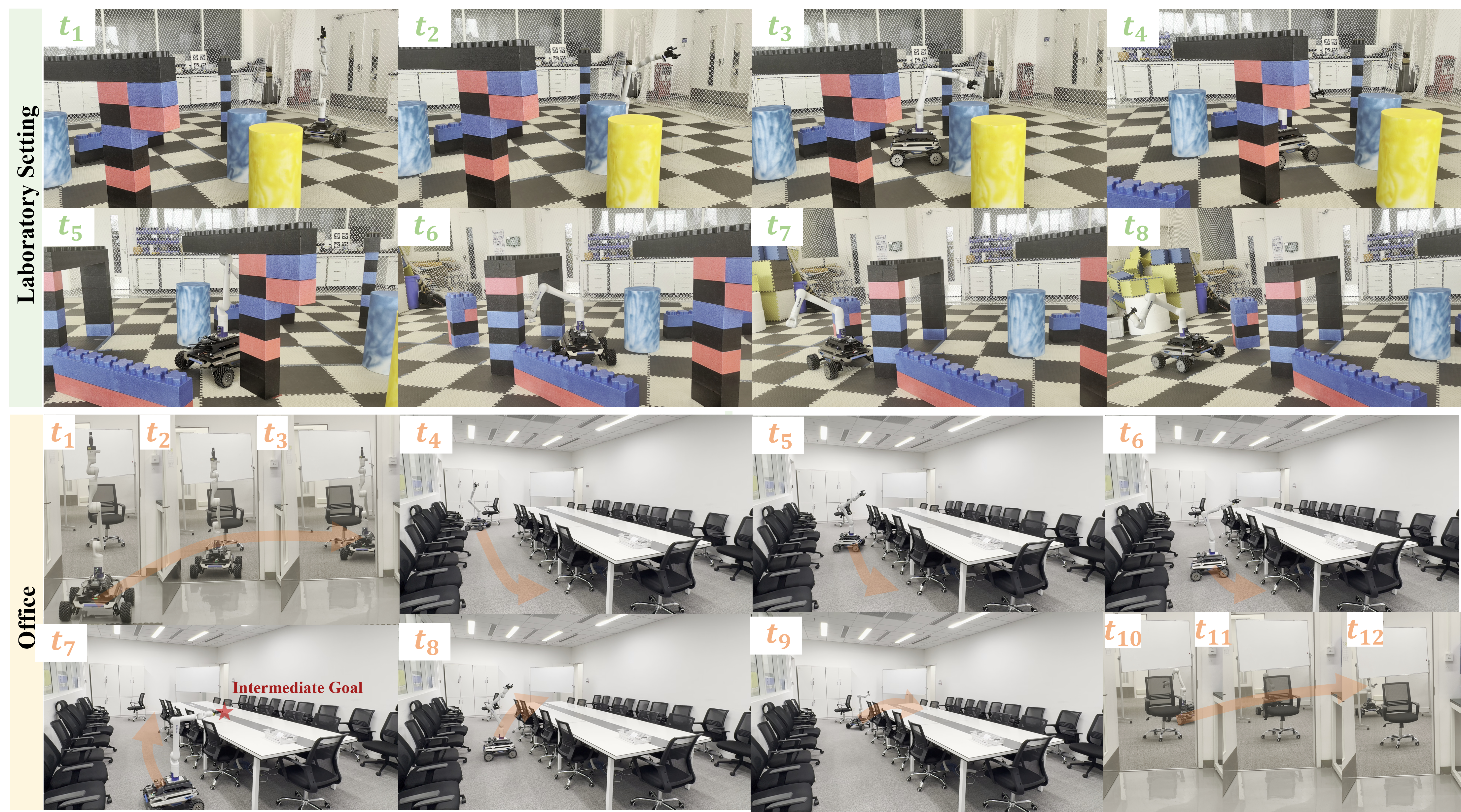}
\caption{Real-world deployment of our planner. \textbf{Top:} in a laboratory obstacle course composed of stacked and overlapping blocks with limited clearance, the robot performs agile base--arm reconfiguration while traversing tight passages ($t_1$--$t_8$). \textbf{Bottom:} in an office environment, the robot follows a multi-goal route (with an intermediate waypoint) through narrow corridors and furniture-dense areas ($t_1$--$t_{12}$). Orange arrows show the direction of motion.}
\label{fig:real_exp}
\end{figure*}

\textbf{Baselines. }
We remark that, under our benchmark setting, most existing methods fail to produce feasible solutions as the environments are deliberately dense and geometrically complex.
We compare our method against four representative baselines. (1) \scenario{RRT}~\cite{sucan2012the-open-motion-planning-library}: a direct sampling-based planner that searches in the full whole-body configuration space and performs collision checking using the environment ESDF. (2) \scenario{LocalDecomp}~\cite{chunxin2025RAL-localreactive}: a SOTA free-space decomposition method that acts as a local reactive controller. At each step, it constructs link-wise safe regions from obstacle information and plans the next motion within these regions to maintain safety. For fairness, we assume the full map is known in advance and thus exclude perception-induced errors. (3) \scenario{REMANI}~\cite{Wu2024ICRA} and (4) \scenario{TopAY}~\cite{xu2025topayefficienttrajectoryplanning}: SOTA two-stage pipelines that decouple base and arm planning in the front end, followed by trajectory optimization in the back end. \scenario{REMANI} plans the base using Hybrid A* and then searches for a discrete arm trajectory along the base path via constrained RRT*-Connect. \scenario{TopAY} first performs parallel topological search on the mobile base to generate multiple independent candidate base references, then refines each candidate by searching nearby arm motions using constrained Bi-RRT*, selecting the best trajectory among candidates with parallel acceleration. Both methods employ penalty-based back-end optimization that augments constraints into the objective as soft penalties, requiring tuning of penalty weights for reliable convergence. For collision avoidance, they rely on the environment ESDF and approximate the robot using a set of collision spheres, minimizing the sum of ESDF values over spheres. This formulation is common in aerial robotics, where the robot can often be approximated as a point (or a single sphere), and it has the practical advantage that the collision-evaluation cost does not grow with obstacle density. Tuning the penalty weights across multiple terms can be challenging for these methods. We therefore adopt a unified strategy that prioritizes collision costs and reports the best performance achieved after careful parameter tuning. The collision spheres are also carefully fitted to closely approximate the robot’s link geometries to avoid conservative behaviors.

Finally, since our goal is a robust and efficient back-end trajectory optimization algorithm that can produce meaningful whole-body trajectories from trivial initial guesses within a limited number of iterations, we evaluate two variants. \textbf{Ours} uses a naive initialization: linear interpolation for the base between start and goal, with the arm fixed at the all-zero (upright) posture. \textbf{Ours (TopAY-init)} augments our GCDF-based solver with the first feasible front-end solution returned by \scenario{TopAY} as initialization, isolating and highlighting the strength of our numerical back end under strong initial guesses. 

\textbf{Main Results. }
We evaluate each method by its success rate (with failure cases including collision, optimization failure, and stagnation) and computation time. To quantify trajectory conservativeness and smoothness, we report the configuration-space path-length ratio over successful trials, defined as each baseline’s trajectory length normalized by that of our method. Representative trajectories produced by our approach are visualized in Fig.~\ref{fig:benchmark_simulation_all}. Quantitative results are summarized in Table~\ref{tab:benchmark_results}, and the distribution of trajectory length ratios is further shown in Fig.~\ref{fig:c_space_ratio}.

\scenario{RRT} consistently failed to find a feasible solution within the allocated time, likely due to the high density of our experimental setup. We therefore omit \scenario{RRT} from Table~\ref{tab:benchmark_results}. The remaining results show that our method efficiently handles large-scale collision constraints across all difficulty levels and consistently achieves the highest success rates. In particular, \textbf{Ours (TopAY-init)} achieves the best robustness with the shortest runtime, underscoring a key practical insight: once the per-iteration cost (constraint evaluation, batching, and QP subproblem construction and solution) is optimized, the end-to-end runtime of an iterative solver is largely determined by the number of iterations required to reach a high-quality feasible solution. This iteration count is influenced not only by initialization quality, but also by the choice of collision constraints: well-posed constraints with informative local first-order geometry can dramatically accelerate convergence. In our case, GCDF enforces collision avoidance directly in configuration space and provides accurate local gradients, yielding more reliable descent directions and faster contraction toward feasibility. A stronger initialization, therefore, further reduces the iteration count and amplifies the benefit of a robust back-end optimizer, while the solver’s refinement capability remains essential for improving trajectory quality.

For two-stage methods such as \scenario{REMANI} and \scenario{TopAY}, the hierarchical front-end search is typically fast when it succeeds; however, as obstacle density increases and the environment becomes more geometrically complex, front-end failures become frequent. In such cases, reverting to a naive initialization often causes the ESDF-sphere, penalty-based back end to stuck or converge to infeasible, collision-prone trajectories. {\scenario{TopAY}'s} parallel candidate strategy partially mitigates this issue by increasing the chance of finding a feasible front-end solution, but it also tends to select more conservative base routes with longer detours in dense clutter.

\scenario{LocalDecomp} exhibits a different failure mode due to its local, reactive nature. Because it reasons only about obstacles in a limited neighborhood, its performance is less sensitive to obstacle count than to global layout complexity; it is particularly prone to getting stuck or colliding in ``bridge-like'' structures and stacked non-convex geometries. Moreover, since it prioritizes local traversability rather than global path quality, it often induces large and unnecessary base and arm rotations.

In contrast, \textbf{Ours} and \textbf{Ours (TopAY-init)} optimize directly in configuration space with implicit GCDF constraints, consistently producing non-conservative whole-body motions in dense 3D clutter. In particular, the naive-initialization variant keeps the base close to a straight-line route and relies on coordinated arm reconfiguration to negotiate tight clearances, resulting in the shortest traversal distance with smooth configuration-space evolution. Fig.~\ref{fig:benchmark_simulation_all} further illustrates these safe and agile behaviors, showing coordinated base--arm maneuvers that maintain clearance without excessive detours. Overall, these results validate the effectiveness of our learned implicit GCDF and training pipeline, together with our customized high-performance numerical algorithm, and highlight the benefit of enforcing collision avoidance directly in configuration space for robust and efficient whole-body trajectory optimization.
\subsection{Real-World Validation}

We further validate the proposed framework on a real mobile manipulator in two representative settings: (\emph{i}) a laboratory obstacle course with manually arranged clutter consisting of stacked blocks of diverse shapes and heights, and (\emph{ii}) an office environment that requires navigation through narrow passages while sequentially reaching multiple intermediate goals. Fig.~\ref{fig:real_exp} visualizes execution with snapshots along the planned trajectories. These deployments are intentionally challenging for whole-body mobile manipulation: the scenes exhibit tight 3D clearances across multiple height layers and non-convex, overlapping geometry (e.g., stacked blocks and bridge-like structures). In the laboratory course (top, $t_1$--$t_8$), the robot navigates through densely placed obstacles with limited clearance. It exhibits non-conservative whole-body reconfiguration: the base commits to narrow passages. At the same time, the arm continuously adjusts its posture to maintain clearance against obstacles at different heights, thereby avoiding excessive detours. In the office scenario (bottom, $t_1$--$t_{12}$), the robot plans and executes a multi-goal route through confined corridors and furniture-dense areas (including an intermediate waypoint), producing smooth collision-free motions without oscillatory corrections or large unnecessary rotations. Across both settings, the executed trajectories maintain consistent clearance in tight spaces, indicating that the learned neural GCDF provides sufficiently accurate local geometry and first-order information to support reliable configuration-space optimization on hardware. Overall, these real-world trials validate the practicality of our approach and its ability to generate safe, agile, and dynamic feasible whole-body motions.

\section{Conclusion}
\label{sec:conclusion}
We presented a configuration-space collision reasoning and trajectory optimization framework for whole-body mobile manipulation in dense, cluttered, and unbounded environments. We introduced GCDF as a generalization of CDF that can go beyond purely rotational manipulators and consider mobile manipulators with both translational and rotational joints. We developed a one-time offline data collection and training pipeline to learn an implicit neural GCDF with accurate values and first-order information. Building on this representation, we proposed a high-performance C++ sequential convex optimization algorithm that natively supports neural implicit constraints through batched GPU queries, sparsity-aware active-constraint selection, and online constraint injection. Extensive randomized clutter benchmarks and real-world experiments demonstrate robust, efficient, and non-conservative whole-body motion generation, highlighting the benefit of enforcing collision avoidance directly in configuration space.

\textbf{Future directions.} First, in the training stage, we did not exhaustively explore how network architecture choices and hyperparameters (e.g., loss-term weights) affect the fidelity of the learned implicit GCDF. In this work, we used a standard MLP; however, more expressive architectures may better fit larger and higher-resolution datasets, yielding more accurate values and gradients and further improving downstream optimization performance. Moreover, the relative weighting between translational and rotational components in the GCDF metric can significantly influence optimization behavior. While we empirically selected a strong setting, an interesting direction is to treat these weights as conditioning inputs to the network, allowing a single model to realize different trade-offs without retraining.

Second, despite substantial efficiency optimizations, the overall computation can still be dominated by the scale of collision constraints, since obstacle geometries are represented by many sampled points. This creates a trade-off between geometric fidelity and constraint complexity: faithfully representing curved or complex obstacles may require dense sampling, whereas sparse sampling can lead to undesired solutions (e.g., the arm threading through gaps between samples). A natural extension is to move beyond point-based constraints by learning GCDF over parametric primitives. For example, by representing obstacles as spheres (or other primitives) with radius $r$ and defining a GCDF that takes $(\bm q,\bm p,r)$ as inputs. Such primitive-aware constraints could significantly reduce the number of constraints, improve geometric coverage, and mitigate failure modes caused by undersampling, while retaining the benefits of configuration-space collision reasoning.

\section*{Acknowledgments}
We thank Long Xu, Zailin Huang, and Chengkai Wu for insightful discussions on experimental settings for mobile manipulators.
\bibliographystyle{IEEEtran}
\bibliography{ref}
\appendices
\section{Proof of \prettyref{thm:gcdf_properties}}
\label{app:gcdf_properties}

\begin{proof}
~\\
\begin{enumerate}
    \item Starting from Definition~\ref{def:cdf_generalized}, let $\bm q_z$ be the closest configuration of $\bm q$ on $\calZ(\bm p)$, we have
    $$
        f^g_c(\bm{p}, \bm{q}) = \sqrt{(\bm q - \bm q_z)^{\top}\bm{M}(\bm q - \bm q_z)},
    $$
    and its partial derivative with respect to $\bm q$:
    $$
        \nabla_{\bm q} f^g_c(\bm{p}, \bm{q}) = \frac{\bm{M}(\bm q - \bm q_z)}{\sqrt{(\bm q - \bm q_z)^{\top}\bm{M}(\bm q - \bm q_z)}} = \frac{\bm{M}(\bm q - \bm q_z)}{f^g_c(\bm{p}, \bm{q})}.
    $$
    Then, with the definition of the weighted norm, we have
    \begin{align*}
        \|\nabla_{\bm{q}} f^g_c(\bm{p}, \bm{q})\|_{\bm{M}^{-1}} &= \sqrt{\nabla_{\bm{q}} f^g_c(\bm{p}, \bm{q})^{\top} \bm{M}^{-1} \nabla_{\bm{q}} f^g_c(\bm{p}, \bm{q})} \\
        &= \sqrt{\frac{(\bm q - \bm q_z)^{\top} \bm{M}^{\top} \bm{M}^{-1} \bm{M} (\bm q - \bm q_z)}{(\bm q - \bm q_z)^{\top} \bm{M} (\bm q - \bm q_z)}} \\
        &= 1.
    \end{align*}

    \item We first illustrate why we can not move along the gradient direction to get to the $\calZ$ in one step. 
    To reach $\calZ$ in one step, we need to find a direction and step size $\lambda$ such that
    $$\bm{q} - \bm{q}_z = \lambda\nabla f^g_c(\bm{p},\bm{q}).$$

    Substituting the gradient expressions, it gives
    $$\bm{q} - \bm{q}_z = \lambda\frac{\bm{M}(\bm{q}-\bm{q}_z)}{\sqrt{(\bm{q}-\bm{q}_z)^T\bm{M}(\bm{q}-\bm{q}_z)}}.$$

    Let $\bm{v} = \bm{q} - \bm{q}_z$. Then, it gives
    $$\bm{v} = \lambda\frac{\bm{M}\bm{v}}{\sqrt{\bm{v}^T\bm{M}\bm{v}}}.$$

    Left multiplying both sides by $\bm{v}^T$, it gives
    $$\bm{v}^T\bm{v} = \lambda\frac{\bm{v}^T\bm{M}\bm{v}}{\sqrt{\bm{v}^T\bm{M}\bm{v}}}.$$

    Therefore, it gives
    $$\lambda = \frac{\bm{v}^T\bm{v}}{\sqrt{\bm{v}^T\bm{M}\bm{v}}} = \frac{\|\bm{q}-\bm{q}_z\|^2}{\sqrt{(\bm{q}-\bm{q}_z)^T\bm{M}(\bm{q}-\bm{q}_z)}}.$$

    Note that, unless $\bm{M}$ is an identity matrix, we cannot obtain $\lambda$ since $\bm{q}_z$ is unknown. However, inspired by the concept of conjugate gradient descent~\cite{Shewchuk1994CGMethod}, we can try moving along the direction of $\bm{M}^{-1}\nabla f^g_c(\bm{p},\bm{q})$ instead:
    $$\bm{q} - \bm{q}_z = \lambda \bm{M}^{-1}\nabla f^g_c(\bm{p},\bm{q}).$$

    Substituting the gradient gives
    $$\bm{q} - \bm{q}_z = \lambda \bm{M}^{-1}\frac{\bm{M}(\bm{q}-\bm{q}_z)}{\sqrt{(\bm{q}-\bm{q}_z)^T\bm{M}(\bm{q}-\bm{q}_z)}}.$$

    Let $\bm{v} = \bm{q} - \bm{q}_z$. Then, we have
    $$\bm{v} = \lambda\frac{\bm{v}}{\sqrt{\bm{v}^T\bm{M}\bm{v}}}.$$

    Since the vectors on both sides are now aligned, we can directly solve for $\lambda$:
    $$\lambda = \sqrt{\bm{v}^T\bm{M}\bm{v}} = \sqrt{(\bm{q}-\bm{q}_z)^T\bm{M}(\bm{q}-\bm{q}_z)} = f^g_c(\bm{p},\bm{q}).$$
\end{enumerate}
\end{proof}

\end{document}